\documentclass[11pt]{article}

\usepackage[preprint]{acl}

\usepackage{times}
\usepackage{latexsym}

\usepackage[T1]{fontenc}

\usepackage[utf8]{inputenc}

\usepackage{microtype}

\usepackage{graphicx}
\usepackage{subcaption}
\usepackage{arydshln} 
\usepackage{booktabs} 
\usepackage{float}
\usepackage{url}
\usepackage{multirow}
\usepackage{siunitx}
\usepackage{longtable} 
\usepackage{pifont}
\newcommand{\cmark}{\ding{51}} 
\newcommand{\xmark}{\ding{55}} 
\usepackage[most]{tcolorbox}
\usepackage{enumitem} 
\usepackage{colortbl}
\usepackage{placeins}
\usepackage{listings}

\usepackage{amsmath,amsfonts,bm}

\def\eqref#1{equation~\ref{#1}}

\def\1{\bm{1}}

\DeclareMathAlphabet{\mathsfit}{\encodingdefault}{\sfdefault}{m}{sl}
\SetMathAlphabet{\mathsfit}{bold}{\encodingdefault}{\sfdefault}{bx}{n}

\definecolor{idealcol}{RGB}{255, 246, 218} 

\newcommand{\revised}[1]{#1}

\usepackage{amsmath}
\usepackage{amssymb}
\usepackage{mathtools}
\usepackage{amsthm}

\usepackage[capitalize,noabbrev]{cleveref}

\usepackage[textsize=tiny]{todonotes}

\theoremstyle{plain}
\newtheorem{theorem}{Theorem}[section]
\newtheorem{proposition}[theorem]{Proposition}

\theoremstyle{definition}

\theoremstyle{remark}

\title{Route Before Retrieve: Activating Latent Routing Abilities of LLMs \\ for RAG vs. Long-Context Selection}

\author{
  Yiwen Chen\textsuperscript{1} \quad
  Kuan Li\textsuperscript{2} \quad
  Fuzhen Zhuang\textsuperscript{1} \quad
  Deqing Wang\textsuperscript{1} \quad
  Zhao Zhang\textsuperscript{1} \\[4pt]
  Liwen Zhang\textsuperscript{3} \quad
  Yong Jiang\textsuperscript{3} \quad
  Shuai Wang\textsuperscript{2} \quad
  Minhao Cheng\textsuperscript{4} \\[6pt]
  \textsuperscript{1}Beihang University \quad
  \textsuperscript{2}HKUST \quad
  \textsuperscript{3}Alibaba Group \quad
  \textsuperscript{4}Pennsylvania State University \\[4pt]
  \texttt{yiwenchen@buaa.edu.cn, likuan@connect.ust.hk, zhuangfuzhen@buaa.edu.cn}
}

\begin{document}

\maketitle

\begin{abstract}
Recent advances in large language models (LLMs) have expanded the context window to beyond 128K tokens, enabling long-document understanding and multi-source reasoning. A key challenge, however, lies in choosing between \textbf{retrieval-augmented generation (RAG)} and \textbf{long-context (LC)} strategies: RAG is efficient but constrained by retrieval quality, while LC supports global reasoning at higher cost and with position sensitivity. Existing methods such as \textit{Self-Route} adopt failure-driven fallback from RAG to LC, but remain passive, inefficient, and hard to interpret. We propose \textbf{Pre-Route}, a proactive routing framework that performs structured reasoning \emph{before} answering. Using lightweight metadata (e.g., document type, length, initial snippet), Pre-Route enables task analysis, coverage estimation, and information-need prediction, producing explainable and cost-efficient routing decisions. Our study shows three key findings: (i) LLMs possess latent routing ability that can be reliably elicited with guidelines, allowing single-sample performance to approach that of multi-sample (Best-of-N) results; (ii) linear probes reveal that structured prompts sharpen the separability of the ``optimal routing dimension'' in representation space; and (iii) distillation transfers this reasoning structure to smaller models for lightweight deployment. Experiments on LaRA (in-domain) and LongBench-v2 (OOD) confirm that Pre-Route outperforms Always-RAG, Always-LC, and Self-Route baselines, achieving superior overall cost-effectiveness.
\end{abstract}

\section{Introduction}
\vspace{-8pt}
In recent years, Large Language Models (LLMs) have shown unprecedented capabilities in processing long contexts. The latest models—such as GPT-5, DeepSeek R1, and Qwen3—support inputs beyond 128K tokens ~\citep{openai2025introducinggpt5,deepseek2025r10528,DBLP:journals/corr/abs-2505-09388_qwen3report}, enabling tasks like long-document comprehension, cross-segment reasoning, and multi-source synthesis. Yet, how to efficiently and accurately exploit these capabilities across tasks remains an open challenge~\citep{DBLP:journals/tacl/LiuLHPBPL24_lostinmiddle}. Two mainstream approaches have emerged: \textbf{Retrieval-Augmented Generation (RAG)}, which retrieves relevant snippets from external sources~\citep{DBLP:conf/nips/LewisPPPKGKLYR020_ragnlp}, and \textbf{Long Context (LC)} processing, where the full context is provided to the model for end-to-end reasoning~\citep{kamradt2023needle,DBLP:journals/corr/abs-2404-06654_ruler}.

These two pathways differ in suitability, each with distinct strengths and limitations. Benchmarks~\citep{DBLP:journals/corr/abs-2502-09977_lara,DBLP:conf/acl/BaiLZL0HDLZHDTL24_longbench,DBLP:conf/acl/AnG0ZLZKQ24_leval} show that LC excels at comparison and complex reasoning by leveraging a global view, while RAG is preferred for fact retrieval and hallucination-sensitive tasks, offering precise evidence and abstention. LC works well for well-structured texts (e.g., financial reports) but suffers from positional effects such as lost-in-the-middle problems ~\citep{DBLP:journals/tacl/LiuLHPBPL24_lostinmiddle}. Conversely, RAG adapts better to loosely structured texts (e.g., novels) and is less position-sensitive, though its effectiveness depends on retrieval quality. Thus, choosing between RAG and LC is a key challenge for advancing long-context processing.

Existing strategies such as Self-Route~\citep{DBLP:conf/emnlp/Li00MB24selfroute} adopt a failure-driven mechanism: attempt RAG first, and switch to LC only if the model outputs "unanswerable". While simple and sometimes effective, this approach has some drawbacks: it is passive, relying solely on retrieval failure signals; it incurs redundant overhead, since RAG retrievals must always be executed; it depends on the model's self-assessment, which may be either over-conservative or overconfident; and it offers little interpretability regarding routing decisions (see Appendix~\ref{app:selfroutefail}).
\begin{figure}[t]
    \centering
    \begin{subfigure}[t]{0.48\linewidth}
        \centering
        \includegraphics[width=\linewidth]{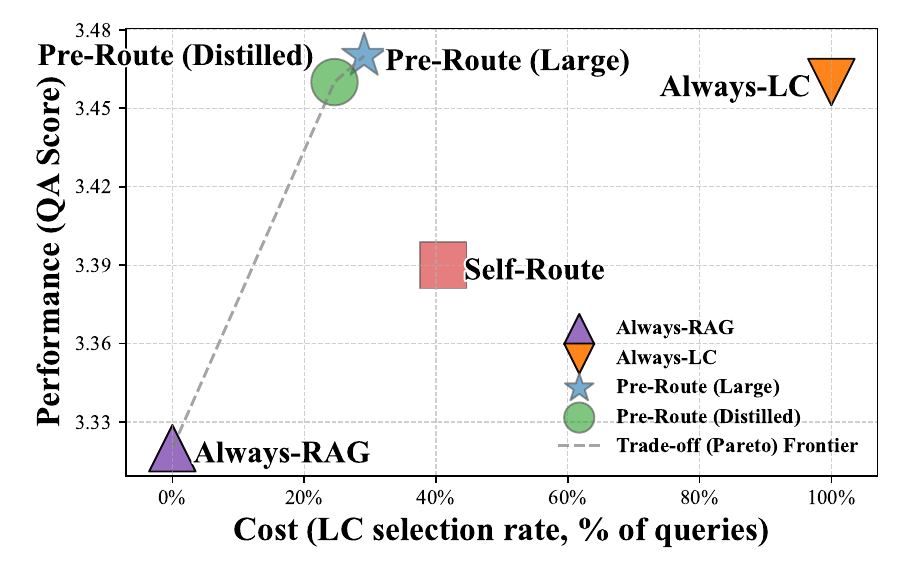}
        \caption{LaRA (in-domain)}
        \label{fig:lara_235b}
    \end{subfigure}
    \hfill
    \begin{subfigure}[t]{0.48\linewidth}
        \centering
        \includegraphics[width=\linewidth]{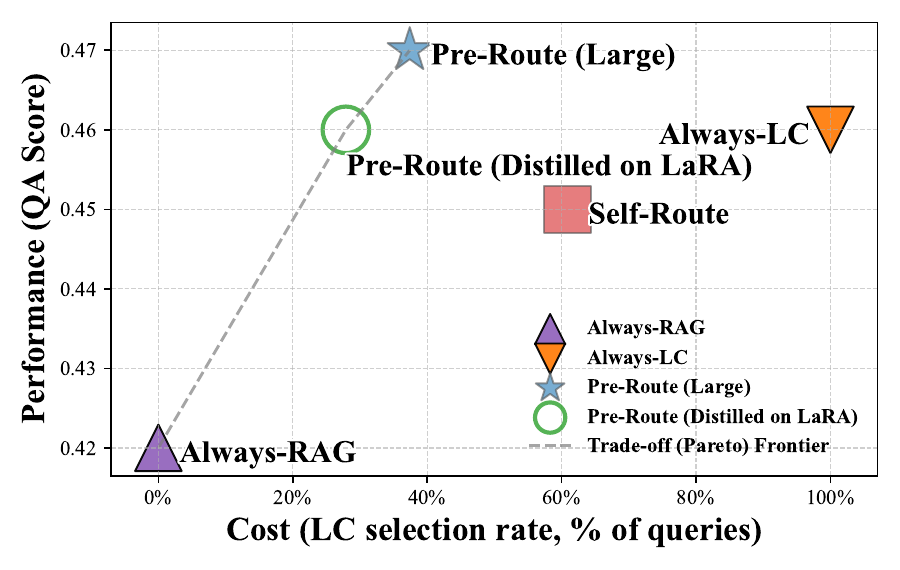}
        \caption{LongBench-v2 (OOD)}
        \label{fig:lb5_235b}
    \end{subfigure}
    \caption{Visualization of Pre-Route cost-effectiveness with Qwen3-235B-A22B:
(a) in-domain (LaRA), (b) out-of-domain (LongBench-v2).}
\vspace{-14pt}
    \label{fig:lara_lb5_compare}
\end{figure}
To overcome these issues, we introduce \textbf{Pre-Route}, a proactive routing framework. Our core insight is that LLMs possess a latent ``routing potential'' dormant in the high-dimensional parameter space. Pre-Route elicits this ability not by teaching new knowledge, but by constraining the decision space through a structured reasoning process, aligning internal attention with critical decision boundaries. Specifically, Pre-Route performs structured reasoning \textit{before} answering using low-cost metadata such as document type, title, length, and initial snippets. Furthermore, this process can be distilled into smaller models for lightweight deployment.

In summary, our contributions are:
\begin{enumerate}
\item We propose Pre-Route, an active and cost-efficient routing framework that introduces structured reasoning before retrieval and answering, implementing a "plan-then-execute" paradigm.
\item We systematically verify and uncover the latent routing capabilities of LLMs. Through BoN experiments and linear probe analysis, we demonstrate the existence of this ability and find that guidelines can effectively elicit and stabilize it, enabling single-sample performance to approach the multi-sample upper bound and forming a more linearly separable ``optimal route dimension'' in the representation space.
\item We successfully transfer this routing capability to smaller models, forging a lightweight, plug-and-play module. We find that while smaller models struggle to acquire this ability through prompting alone, they can effectively learn the Pre-Route reasoning patterns via distillation. This paves the way for efficient, low-cost deployment on edge and client-side devices.
\item We achieve significant advantages in both in-distribution and out-of-distribution (OOD) tasks. On the LaRA Benchmark and LongBench-v2 datasets, Pre-Route demonstrates superior effectiveness and generalization by outperforming Self-Route in both performance and efficiency.
\end{enumerate}

\vspace{-12pt}

\section{From Dormant to Active: Unlocking Latent Routing Potential}
\label{sec:latent}
\vspace{-8pt}
Existing methods like Self-Route ~\citep{DBLP:conf/emnlp/Li00MB24selfroute} treat routing as a retrieval fallback—RAG is always attempted first—implicitly assuming LLMs lack intrinsic planning. In this work, we challenge this assumption and ask a more fundamental question: \textbf{Do LLMs already possess a "latent routing ability"—the internal competence to discern task requirements and select an appropriate path?} To address this question, we adopt two complementary perspectives: (i) \textbf{behavioral experiments} that demonstrate both the existence of such latent ability and its instability, and (ii) \textbf{representation analysis} that probes the model's internal mechanisms to reveal how this ability can be elicited. Taken together, our argument is that routing ability is not absent but rather dormant, and can be reliably elicited by suitable guidelines. \revised{The routing guidelines were derived from prior studies and finalized on the LaRA training set; the test set remained unseen during prompt construction. All analyses below are post-hoc verifications conducted after prompts and models were frozen. Experimental setup details are in Sec.~\ref{sec:exp_setup}.}
\vspace{-12pt}

\begin{figure*}[tbp]
\vspace{-12pt}
    \centering
    \includegraphics[width=0.75\linewidth]{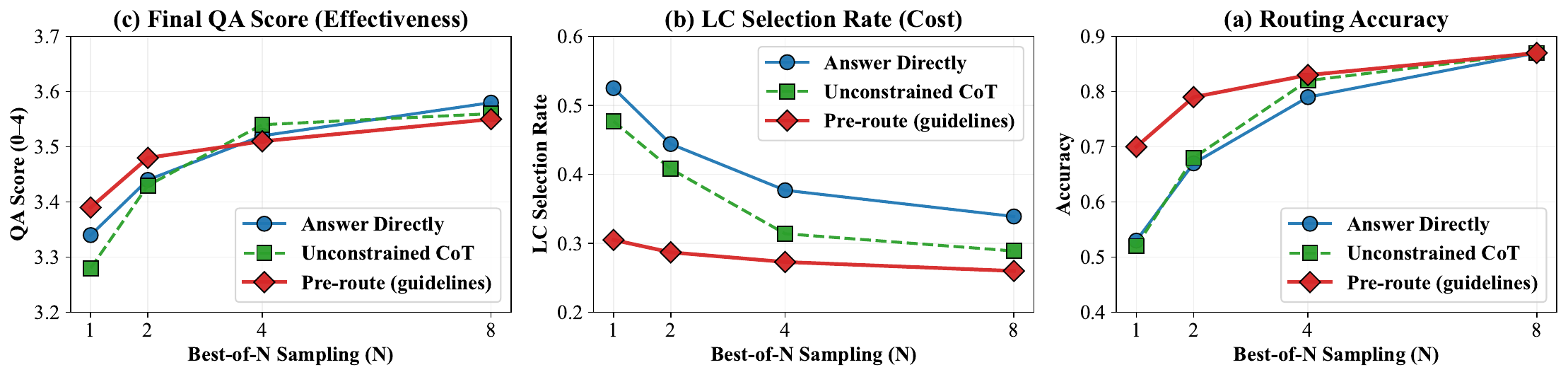}
    \vspace{-10pt}
    \caption{BoN comparison results tested with answer model Qwen3-235B-A22B.}
    \label{fig:bon_comparison}
\vspace{-14pt}
\end{figure*}

\subsection{Behavioral Evidence: Revealing Latent Ability via Best-of-N Sampling}

To disentangle whether routing failures arise from a lack of capacity or merely insufficient activation, we recast routing as a binary classification task (predicting RAG vs.\ LC). We adopt the \textbf{Best-of-N (BoN) sampling} methodology widely used for testing latent abilities.

We design three controlled comparisons \revised{under identical metadata}:
(1) Prompt Paradigm: \emph{answer directly}, \emph{unconstrained CoT} (free-form reasoning before routing), and our proposed \emph{Pre-Route} (structured reasoning guidelines).
(2) Sampling Strategy: $N \in \{1,2,4,8\}$.
(3) Metrics: QA Score, LC selection rate and routing accuracy.  We report \textbf{routing accuracy}, defined as the proportion of model decisions that match an \emph{ideal label}—the decision rule that maximizes QA performance while defaulting to the lower-cost RAG option when performances are comparable (formalized later in Sec.~\ref{sec:learn}).

Experimental results (Fig.~\ref{fig:bon_comparison}) clearly reveal the model's internal decision capacity. Under \emph{Answer Directly} and \emph{Unconstrained CoT}, routing accuracy shows a strong positive correlation with $N$. For instance, in the Answer Directly setting, accuracy rises from about 0.53 (N=1) to 0.87 (N=8), indicating that the knowledge required for correct routing is present but accessed stochastically: sometimes the model follows the correct path, other times it diverges. Increasing the number of samples simply raises the chance of hitting the right path, thereby exposing the model's latent ability.

In sharp contrast, introducing the \emph{Pre-Route (guidelines)} structured reasoning chain fundamentally changes the picture. Its curve is much flatter: accuracy already reaches 0.70 at N=1---substantially higher than the other two settings---and further climbs to 0.83 at N=4 before saturating at 0.87 by N=8. This pattern indicates that \textbf{structured guidance acts as a calibrator and stabilizer}. It does not inject new knowledge, but rather provides a clear reasoning scaffold that reliably elicits and directs the model's latent routing capability, enabling single-shot performance close to the model's upper bound.

\subsection{Representation Evidence: Probing the Decision Space}

\begin{table}[t]
  \caption{Linear probe accuracy tested with answer model Qwen3-235B-A22B.}
  \vspace{-8pt}
\centering
\small
\resizebox{\linewidth}{!}{
        \begin{tabular}{c c c
                        >{\columncolor{idealcol}}S[table-format=1.4]
                        S[table-format=1.4]
                        S[table-format=1.4]
                        S[table-format=1.4]}
        \toprule
        \multirow{2}{*}{\textbf{Model}} &
        \multirow{2}{*}{\textbf{Prompt}} &
        \multirow{2}{*}{\textbf{Distilled}} &
        \multicolumn{4}{c}{\textbf{Accuracy}} \\
        \cmidrule(l){4-7}
         & & & {\cellcolor{idealcol}Ideal} & {Route} & {Doc} & {Task} \\
        \midrule
        Qwen3-1.7B & Pre-Route       & \cmark & \bfseries \textbf{0.6389} & \bfseries \textbf{0.7986} & \textbf{0.3958} & 0.4097 \\
        Qwen3-1.7B & Pre-Route       & \xmark & 0.6250 & 0.7639 & 0.3330 & \textbf{0.4167} \\
        Qwen3-8B   & answer direct.  & \xmark & 0.5486 & 0.6597 & 0.2986 & 0.2569 \\
        Qwen3-1.7B & answer direct.  & \xmark & 0.5208 & 0.6389 & 0.3194 & 0.2500 \\
        Qwen3-8B   & unconstr.\ CoT  & \xmark & 0.4653 & 0.6944 & 0.3681 & 0.2500 \\
        Qwen3-1.7B & unconstr.\ CoT  & \xmark & 0.3958 & 0.5764 & 0.3472 & 0.2639 \\
        \bottomrule
        \end{tabular}}

\label{tab:linear_probe}
\vspace{-15pt}
\end{table}

While behavioral results confirm that routing ability can be elicited, an open question is what changes occur internally. We hypothesize that structured prompting not only alters outputs, but also reorganizes the representation space, such that routing signals become more linearly separable. To investigate this, we apply linear probes on frozen representations and examine four prediction targets: (i) the \emph{ideal} decision, defined in Sec.~\ref{sec:learn} as the \emph{task-optimal decision}—the choice that maximizes QA performance while defaulting to RAG when tied; (ii) the model's own \emph{route} choice, reflecting its internal bias; (iii) \emph{doc\_type} (7 classes), to test whether routing simply mirrors document categories; and (iv) \emph{task\_type} (4 classes), to test whether routing is merely driven by shallow task-level hints rather than deeper reasoning signals. Among these, ideal label accuracy is the most important indicator, since it directly evaluates whether the model's latent routing aligns with task-optimal decisions.

We apply \emph{linear probing}~\citep{DBLP:conf/iclr/AlainB17_linearprobe}: a linear classifier trained on frozen penultimate last token embeddings to predict the labels above. Linear probes are deliberately chosen since linear separability implies explicit encoding in the representation. 1) \textbf{Prompt structure is critical:} on Qwen3-1.7B, ideal-label accuracy improves from 0.396 (unconstr.\ CoT) to 0.521 (direct) to \textbf{0.625} (Pre-Route). 2) \textbf{Distillation further aligns:} adding distillation(see Sec.~\ref{sec:learn}) on Pre-Route raises accuracy to 0.639. 3) \textbf{Design over scale:} Qwen3-8B with direct prompting (0.549) underperforms Qwen3-1.7B with Pre-Route (0.625). 4) \textbf{Model's own routing decision:} under structured guidance, the model's representation of its own decisions becomes clearer, making probes better at predicting its \emph{route} choice (0.58--0.80), though not always aligned with the task-optimal choice. 5) \textbf{Not shallow heuristics:} low accuracy on \emph{doc\_type} and \emph{task\_type} shows routing is not reducible to trivial document or task categories.
Together, these findings confirm that optimal routing boundaries are already encoded in the representation space. BoN experiments demonstrate it behaviorally, while probing confirms it representationally. Crucially, structured guidance and light supervision reliably elicit and stabilize this ability—laying the foundation for the Pre-Route framework introduced next.

\begin{figure*}[!t]
\vspace{-8pt}
    \centering
    \includegraphics[width=0.95\linewidth]{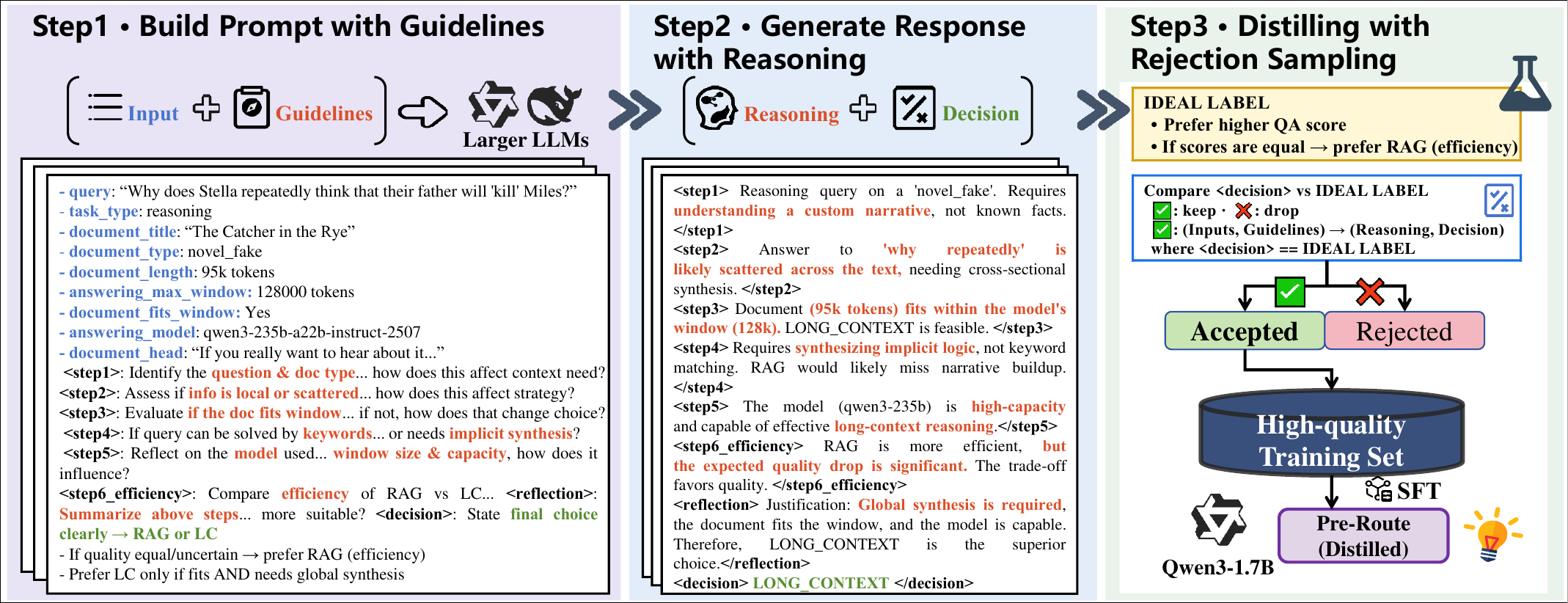}
    \vspace{-12pt}
    \caption{Overview of the Pre-Route framework. Step 1 builds structured prompts with meta-infos and guidelines; Step 2 generates reasoning-based routing decisions; Step 3 performs rejection sampling to filter out suboptimal decisions, distilling high-quality training data for efficient and interpretable routing with smaller models (Qwen3-1.7B). (Complete prompt is provided in Fig.~\ref{fig:complete_decision_prompt})}
    \label{fig:main}
\vspace{-14pt}
\end{figure*}

\vspace{-8pt}

\section{Methodology}
\vspace{-8pt}

\subsection{Problem Formalization and Theoretical Framework}

\paragraph{Decision-theoretic formalization.}
Given a query \(q\) and long document \(\mathcal D\), a routing policy \(\pi\) selects \(y\in\{RAG,LC\}\) solely from low-cost meta-information \(m\) (title, type, length estimate, leading snippet, etc.), leveraging the observation that \(m\) often encodes sufficient cues about information distribution and task difficulty. Formally, we let \(U(y;q,\mathcal D)\) denote the expected QA performance of decision \(y\), and \(C(y;q,\mathcal D)\) denote the corresponding execution cost.

\vspace{-6pt}
\subsection{The Pre-Route Framework}
\vspace{-2pt}

\paragraph{Core hypothesis and idea.}
LLMs already encode task understanding and information-need sensing, yet standard pipelines do not explicitly \emph{activate} this capacity. Pre-Route triggers the latent planning ability by performing \textbf{structured reasoning with decision labels} \emph{before} answering: (i) a \textbf{pre-decision} is made prior to costly full-document processing; (ii) the decision relies on \textbf{meta-information only}, avoiding full-document retrieval overhead; (iii) a \textbf{structured rationale} provides interpretability. In practice, Pre-Route uses only low-cost \(m\): user query, task type, document title/type, document length, answering model, a short document head, and RAG configuration; no retriever or answering model is called, markedly reducing planning overhead \(C_{\text{route}}\) (see cost analysis below).

\vspace{-10pt}
\paragraph{Structured reasoning-chain design.}
Routing is conceptually grounded in four broad \emph{decision dimensions}—(i) task \& document characterization, i.e., how task or document types affect the need for retrieval vs.\ long-context; (ii) information distribution, i.e., whether relevant content is missing, fragmented, or dispersed across sections; (iii) capability–constraint matching, i.e., whether retrieval suffices or long-context is required to avoid positional or window-size limits; and (iv) decision with explanation, i.e., outputting the route together with rationale and fallback considerations.
In practice, these dimensions are instantiated through six \emph{steps} (Fig.~\ref{fig:main}): \textbf{(1) task \& document characterization}; \textbf{(2) distribution pattern judgment}; \textbf{(3) context-window feasibility}; \textbf{(4) retrieval feasibility}; \textbf{(5) model capability consideration}; and \textbf{(6) efficiency trade-off}. This layered design aligns with prior empirical findings \citep{DBLP:conf/iclr/0008PWM0LSBSC24_raglc, DBLP:conf/emnlp/Li00MB24selfroute, DBLP:journals/corr/abs-2502-09977_lara} through a structured pipeline, highlighting operational and transparent decision process while yielding interpretability, efficiency, and robust generalization.

\vspace{-8pt}
\subsection{Cost Analysis}
\label{sec:cost}

We decompose cost into a routing part and an answering part:
\vspace{-8pt}
\begin{equation}
C(y;q,\mathcal D)=C_{\text{route}}(m)+C_{\text{answer}}(y;q,\mathcal D),
\label{eq:cost-decompose}
\end{equation}
where \(C_{\text{route}}(m)\) is the cost of making route decision, and \(C_{\text{answer}}(y;q,\mathcal D)\) corresponds to downstream execution (retrieval for RAG vs.\ long-context processing for LC). Using this decomposition, the expected total cost of a routing policy with LC-selection probability \(p(\text{LC})=\Pr(y=LC)\) is
\vspace{-8pt}
\begin{equation}
C=C_{\text{route}} + p(\text{LC})\cdot C_{\text{LC}} + \big(1-p(\text{LC})\big)\cdot C_{\text{RAG}},
\label{eq:total-cost}
\end{equation}
where \(C_{\text{LC}}\) and \(C_{\text{RAG}}\) denote the average answering costs under LC and RAG respectively. In what follows we analyze \(C_{\text{route}}\) and \(C_{\text{answer}}\) separately and then combine the two perspectives.

\subsubsection{Analysis of \(C_{\text{route}}\): Self-Route vs.\ Pre-Route}
\begin{table}[t]
\centering
\caption{Per-decision routing cost in USD (average query\revised{, computed over all queries in the LaRA benchmark}).}
\resizebox{\linewidth}{!}{ 
\begin{tabular}{lcccccc}
\toprule
Method & Model & Input Toks & Output Toks & \multicolumn{3}{c}{Cost ($\times 10^{-3}$ USD)} \\
\cmidrule(lr){5-7}
 &  &  &  & Input & Output & \textbf{Total} \\
\midrule
Self-Route & Qwen3-235B & 2600 & 27  & 0.73 & 0.03 & 0.76 \\
Pre-Route  & Qwen3-235B & 1205 & 648 & 0.34 & 0.73 & 1.07 \\
Pre-Route  & Qwen3-1.7B & 1205 & \revised{670} & 0.05 & 0.11 & \textbf{0.16} \\
\bottomrule

\end{tabular}}
\label{tab:route-cost}
\vspace{-16pt}
\end{table}
Pre-Route achieves lower routing cost than the Self-Route baseline once distillation is applied: with Qwen3-1.7B the cost is reduced to about one-fifth that of Self-Route (Tab.~\ref{tab:route-cost}). With Qwen3-235B, the cost is slightly higher due to the generated rationale, but this overhead is offset once a smaller router is used. This demonstrates that planning overhead can be kept low while retaining interpretability.

\subsubsection{Analysis of \(C_{\text{answer}}\): LC vs.\ RAG}
Answering cost is the dominant component, and it is largely governed by the LC selection probability \(p(\text{LC})\). A long-context (LC) pass processes the entire document, whereas retrieval-augmented generation (RAG) operates only on retrieved passages. Since \(C_{\text{LC}}\) is substantially larger than \(C_{\text{RAG}}\), even moderate changes in \(p(\text{LC})\) have a significant impact on total cost. Hence, controlling \(p(\text{LC})\) is central to cost management.

\subsubsection{Overall cost comparison}
When both components are combined, routing overhead \(C_{\text{route}}\) remains small relative to answering cost. Even with a large router (235B), planning cost is below 4\% of a single 100k-token LC pass (\$0.028); with a small router (1.7B), the proportion falls below 1\%. These comparisons indicate that the total cost is primarily determined by the answering term, and in particular by the LC selection probability \(p(\text{LC})\).
Critically, Pre-Route provides \emph{structural} savings over post-retrieval routing (e.g., Self-Route): since the routing decision is made \emph{before} document processing, the retrieval overhead (embeddings, vector DB queries, reranking) is incurred only when RAG is selected, rather than unconditionally for every query.

\subsection{Learning Paradigm: Selective Strategy Alignment via Distillation}

\label{sec:learn}
\paragraph{Motivation and cost optimization.}
The overall cost is primarily determined by the answering component, and in particular by the LC selection probability \(p(\text{LC})\). Yet Routing overhead \(C_{\text{route}}\) still matters for latency and scalability in practical deployment. Large models (e.g., Qwen3-235B) can produce strong routing in zero-shot settings, yet the planning overhead \(C_{\text{route}}\) is non-trivial. We distill the planning ability into smaller models (e.g., Qwen-7B/1.7B) to reduce cost and latency, improve deployability in constrained environments, and specialize the student for routing.

\paragraph{Learning targets.}
For fixed meta-information \(m\), the objective is to maximize expected QA performance by choosing the option with higher performance. When utilities are comparable or the difference is marginal (e.g., tied scores in multi-level or binary evaluation), the lower-cost RAG is preferred, since \(C_{\text{RAG}} \ll C_{\text{LC}}\). This principle leads directly to the definition of the \textbf{ideal label} for each sample:
\vspace{-8pt}
\begin{equation}
\hat{y}_{\text{ideal}}=
\begin{cases}
LC, & U(LC;q,\mathcal D) > U(RAG;q,\mathcal D),\\
RAG, & \text{otherwise}.
\end{cases}
\label{eq:ideal-label}
\end{equation}

Intuitively, LC is chosen only when its performance advantage over RAG; when utilities are comparable, the lower-cost RAG is preferred. Improving routing accuracy reduces expected performance loss by avoiding wasteful LC on unhelpful samples and reallocating budget to cases that truly benefit from LC. We formally prove in Appendix~\ref{app:proof-cost-minimal} that optimizing the ideal label ensures minimal expected cost among all strategies consistent with performance-optimal decisions.

\paragraph{Stage 1: Rejection sampling.}
Let the teacher \(\pi_T\) produce candidates \(S_i=(T_i,y_i)\); keep entries where the teacher decision matches the ideal rule,
\vspace{-8pt}
\begin{equation}
\mathcal D_{\text{filtered}}=\{(m_i,T_i,y_i)\mid y_i=\hat{y}_{\text{ideal},i}\},
\label{eq:filtered}
\end{equation}
\revised{where \(m_i\), \(T_i\), \(y_i\), and \(\hat{y}_{\text{ideal},i}\) denote the metadata, reasoning trace, decision, and ideal decision for example \(i\)}, thus ensuring data quality, correcting the empirical distribution toward a near-optimal support, removing teacher errors, and enforcing consistency with task-optimal supervision.

\paragraph{Stage 2: Path SFT (student fitting on successful support).}
Train the student on \(\mathcal D_{\text{filtered}}\) by minimizing
\vspace{-8pt}
\begin{equation}
\ \mathcal L_{\text{SFT}}(\theta_S)=-\mathbb E_{(m,T,y)\sim \mathcal D_{\text{filtered}}}\big[\log \pi_S(T,y\mid m)\big]\ ,
\label{eq:sft-loss}
\end{equation}
which is equivalent to minimizing \(\mathrm{KL}(q^*\Vert \pi_S)\) where \(q^*\) is the teacher's \emph{successful} sub-distribution. Unlike answer distillation, this transfers \emph{how to reason} rather than merely \emph{what to answer}.
\vspace{-8pt}
\paragraph{Data construction strategy.}
We generate reasoning chains with Qwen3-235B-A22B and DeepSeek-R1; assign \(\hat{y}_{\text{ideal}}\) consistently; enhance scenario diversity by mixing answering-model scales; perform a global 70/10/20 train/val/test split stratified by (type, level, context length); then merge and retain only training entries where the teacher is consistent with the ideal label.\label{sec:data_construct}

\vspace{-10pt}

\section{Experiments}
\vspace{-8pt}

\begin{table*}[t]
\vspace{-10pt}
\centering
\large

\caption{LaRA (In-distribution main experiment): Pre-Route vs. Self-Route and fixed baselines; Best and Second per answer model. \emph{Pre-Route Variants:} Large: R1 (DeepSeek-R1) or Q235B (Qwen3-235B); Small: Q1.7B (Qwen3-1.7B); Distilled: D-Q1.7B (Distilled-Qwen3-1.7B).}
\vspace{-12pt}
\begin{tcolorbox}[colback=red!4, frame hidden, sharp corners, boxrule=0pt, left=1pt, right=1pt, top=1pt, bottom=1pt, before skip=0pt, after skip=0pt]
\scriptsize Red-shaded entries denote our recommended \textbf{Pre-Route(Large)} and \textbf{Pre-Route(Distilled)}.
\end{tcolorbox}
\label{tab:wide-In-distribution-larawide}
\resizebox{0.92\textwidth}{!}{%
\begin{tabular}{l|rrr|rrr|rrr|rrr|rrr|rrr}
\toprule
\multirow{2}{*}{\begin{tabular}[c]{@{}l@{}}\textbf{Answer Model}\\\textbf{Router Model}\end{tabular}} & \multicolumn{3}{c|}{Qwen3-1.7B [N]} & \multicolumn{3}{c|}{Qwen3-1.7B [T]} & \multicolumn{3}{c|}{Qwen3-4B [N]} & \multicolumn{3}{c|}{Qwen3-4B [T]} & \multicolumn{3}{c|}{Qwen3-8B [N]} & \multicolumn{3}{c}{Qwen3-8B [T]} \\
 & \textbf{QA$\uparrow$} & \textbf{LC(\%)$\downarrow$} & \textbf{Acc$\uparrow$} & \textbf{QA$\uparrow$} & \textbf{LC(\%)$\downarrow$} & \textbf{Acc$\uparrow$} & \textbf{QA$\uparrow$} & \textbf{LC(\%)$\downarrow$} & \textbf{Acc$\uparrow$} & \textbf{QA$\uparrow$} & \textbf{LC(\%)$\downarrow$} & \textbf{Acc$\uparrow$} & \textbf{QA$\uparrow$} & \textbf{LC(\%)$\downarrow$} & \textbf{Acc$\uparrow$} & \textbf{QA$\uparrow$} & \textbf{LC(\%)$\downarrow$} & \textbf{Acc$\uparrow$} \\
\cmidrule(l{-1pt}r{-1pt}){1-19}
\cmidrule(l{-1pt}r{-1pt}){1-19}
Always-LC (Baseline)  & 2.13 & 100.0 & 0.18 & 2.29 & 100.0 & 0.18 & 2.89 & 100.0 & 0.31 & 3.14 & 100.0 & 0.36 & 2.98 & 100.0 & 0.31 & 3.23 & 100.0 & 0.35 \\
Always-RAG (Baseline)  & 2.70 & 0.0 & 0.82 & 2.89 & 0.0 & 0.82 & 3.08 & 0.0 & 0.69 & 3.17 & 0.0 & 0.64 & 3.13 & 0.0 & 0.69 & 3.22 & 0.0 & 0.65 \\
Self-Route (Baseline) & 2.22 & 33.6 & 0.49 & 2.40 & 31.7 & 0.48 & 2.81 & 24.6 & 0.53 & 3.06 & 30.5 & 0.47 & 3.04 & 28.1 & 0.55 & 3.23 & 30.7 & 0.57 \\
\cmidrule(l{-1pt}r{-1pt}){1-19}
\rowcolor{red!4}
Pre-Route (R1) [T] & \underline{2.70} & \underline{2.7} & \underline{0.83} & 2.88 & \textbf{1.8} & \underline{0.83} & 3.09 & \textbf{5.1} & \textbf{0.73} & 3.20 & \textbf{4.8} & \underline{0.70} & 3.15 & \textbf{5.0} & \textbf{0.76} & 3.23 & \textbf{8.0} & \underline{0.70} \\
\rowcolor{red!4}
Pre-Route (Q235B) [N] & \textbf{2.71} & 4.7 & \underline{0.83} & \underline{2.90} & 4.8 & 0.82 & \textbf{3.13} & 21.1 & \textbf{0.73} & \textbf{3.26} & 26.6 & \underline{0.70} & 3.16 & 22.2 & 0.73 & \textbf{3.30} & 26.2 & \underline{0.70} \\
\rowcolor{red!4}
Pre-Route (Q235B) [T] & \textbf{2.71} & \textbf{2.3} & \textbf{0.84} & \textbf{2.91} & 3.0 & \textbf{0.84} & \underline{3.11} & \underline{13.1} & \textbf{0.73} & \underline{3.22} & \underline{15.7} & \underline{0.70} & \underline{3.17} & \underline{15.9} & \underline{0.74} & 3.27 & \underline{17.9} & \textbf{0.71} \\
Pre-Route (Q1.7B) [N] & \underline{2.70} & 7.9 & 0.79 & 2.85 & 10.8 & 0.77 & 3.04 & 30.1 & 0.61 & 3.20 & 35.9 & 0.57 & 3.12 & 34.0 & 0.60 & \underline{3.29} & 33.2 & 0.63 \\
Pre-Route (Q1.7B) [T] & 2.68 & 10.8 & 0.77 & 2.87 & 11.1 & 0.80 & 3.07 & 30.8 & 0.62 & \underline{3.22} & 40.1 & 0.58 & 3.10 & 35.6 & 0.61 & \underline{3.29} & 39.9 & 0.58 \\
\rowcolor{red!4}
Pre-Route (D-Q1.7B) [N] & \underline{2.70} & 3.6 & 0.82 & 2.89 & 3.9 & \underline{0.83} & 3.10 & 17.9 & 0.71 & \textbf{3.26} & 21.2 & \textbf{0.71} & 3.16 & 21.5 & 0.73 & \textbf{3.30} & 21.4 & \textbf{0.71} \\
\rowcolor{red!4}
Pre-Route (D-Q1.7B) [T] & \textbf{2.71} & 3.2 & \underline{0.83} & 2.89 & \underline{2.8} & 0.82 & 3.10 & 21.1 & \underline{0.72} & \textbf{3.26} & 20.5 & \underline{0.70} & \textbf{3.19} & 19.5 & \textbf{0.76} & \textbf{3.28} & 20.8 & \underline{0.70} \\
\cmidrule(l{-1pt}r{-1pt}){1-19}
\cmidrule(l{-1pt}r{-1pt}){1-19}
\multirow{2}{*}{\begin{tabular}[c]{@{}l@{}}\textbf{Answer Model}\\\textbf{Router Model}\end{tabular}} & \multicolumn{3}{c|}{Qwen3-30B [N]} & \multicolumn{3}{c|}{Qwen3-30B [T]} & \multicolumn{3}{c|}{Qwen3-235B [N]} & \multicolumn{3}{c|}{Qwen3-235B [T]} & \multicolumn{3}{c|}{DeepSeek-R1 [T]} & \multicolumn{3}{c}{Qwen-Max [N]} \\
 & \textbf{QA$\uparrow$} & \textbf{LC(\%)$\downarrow$} & \textbf{Acc$\uparrow$} & \textbf{QA$\uparrow$} & \textbf{LC(\%)$\downarrow$} & \textbf{Acc$\uparrow$} & \textbf{QA$\uparrow$} & \textbf{LC(\%)$\downarrow$} & \textbf{Acc$\uparrow$} & \textbf{QA$\uparrow$} & \textbf{LC(\%)$\downarrow$} & \textbf{Acc$\uparrow$} & \textbf{QA$\uparrow$} & \textbf{LC(\%)$\downarrow$} & \textbf{Acc$\uparrow$} & \textbf{QA$\uparrow$} & \textbf{LC(\%)$\downarrow$} & \textbf{Acc$\uparrow$} \\
\cmidrule(l{-1pt}r{-1pt}){1-19}
\cmidrule(l{-1pt}r{-1pt}){1-19}
Always-LC (Baseline)  & 3.37 & 100.0 & 0.40 & 3.39 & 100.0 & 0.39 & 3.46 & 100.0 & 0.34 & 3.51 & 100.0 & 0.40 & 3.44 & 100.0 & 0.35 & 3.36 & 100.0 & 0.39 \\
Always-RAG (Baseline)  & 3.18 & 0.0 & 0.60 & 3.27 & 0.0 & 0.61 & 3.32 & 0.0 & 0.66 & 3.33 & 0.0 & 0.60 & 3.38 & 0.0 & 0.65 & 3.20 & 0.0 & 0.61 \\
Self-Route (Baseline) & 3.20 & 33.9 & 0.53 & 3.10 & 35.5 & 0.42 & 3.39 & 41.1 & 0.56 & 3.34 & 33.9 & 0.52 & 3.36 & 31.4 & 0.52 & 3.28 & 36.5 & 0.56 \\
\cmidrule(l{-1pt}r{-1pt}){1-19}
\rowcolor{red!4}
Pre-Route (R1) [T] & 3.26 & \textbf{14.9} & \textbf{0.69} & 3.30 & \textbf{9.6} & \underline{0.69} & 3.39 & \textbf{14.2} & \underline{0.73} & 3.37 & \textbf{10.8} & \underline{0.68} & 3.42 & \textbf{10.0} & 0.70 & 3.24 & \textbf{10.6} & \underline{0.67} \\
\rowcolor{red!4}
Pre-Route (Q235B) [N] & \textbf{3.31} & 33.8 & 0.65 & \textbf{3.39} & 29.4 & 0.68 & \textbf{3.47} & 29.1 & 0.72 & \underline{3.43} & 27.2 & 0.67 & 3.47 & 27.2 & 0.70 & \textbf{3.31} & 23.1 & \textbf{0.69} \\
\rowcolor{red!4}
Pre-Route (Q235B) [T] & 3.27 & \underline{17.9} & \underline{0.68} & 3.32 & \underline{18.6} & \underline{0.69} & 3.42 & \underline{20.2} & 0.72 & 3.40 & \underline{18.3} & \underline{0.68} & 3.45 & \underline{16.9} & 0.71 & \underline{3.29} & \underline{19.6} & \textbf{0.69} \\
Pre-Route (Q1.7B) [N] & 3.26 & 42.3 & 0.53 & \underline{3.37} & 37.2 & 0.57 & 3.36 & 40.9 & 0.53 & 3.41 & 40.6 & 0.54 & 3.44 & 32.5 & 0.56 & 3.25 & 36.6 & 0.55 \\
Pre-Route (Q1.7B) [T] & 3.27 & 40.0 & 0.61 & 3.36 & 38.5 & 0.59 & 3.41 & 35.6 & 0.59 & 3.41 & 33.1 & 0.59 & \underline{3.49} & 31.5 & 0.68 & 3.26 & 28.3 & 0.60 \\
\rowcolor{red!4}
Pre-Route (D-Q1.7B) [N] & 3.28 & 27.7 & 0.66 & \textbf{3.39} & 28.0 & \textbf{0.70} & \underline{3.46} & 24.6 & \textbf{0.74} & \underline{3.43} & 22.7 & \textbf{0.69} & \textbf{3.51} & 20.6 & \textbf{0.73} & 3.28 & 24.7 & \underline{0.67} \\
\rowcolor{red!4}
Pre-Route (D-Q1.7B) [T] & \underline{3.29} & 26.4 & 0.65 & 3.35 & 24.3 & 0.67 & \textbf{3.47} & 26.9 & \underline{0.73} & \textbf{3.44} & 26.0 & 0.67 & 3.47 & 20.3 & \underline{0.72} & \textbf{3.31} & 24.5 & \textbf{0.69} \\
\cmidrule(l{-1pt}r{-1pt}){1-19}
\end{tabular}}

\vspace{-12pt}
\end{table*}

\subsection{Experimental Setup}
\label{sec:exp_setup}

\paragraph{Datasets and Benchmarks}
We evaluate Pre-Route on both in-domain and out-of-domain settings. The in-domain dataset is \textbf{LaRA}, used for training and primary evaluation. To assess generalization, we further evaluate on the out-of-domain benchmark \textbf{LongBench-v2}.  For \textbf{LaRA}, since knowledge distillation is involved during data construction (see Sec.~\ref{sec:data_construct}), we adopt stratified random splits by \emph{document type}, \emph{task level}, and \emph{context length} for all experiments to ensure fair comparison and prevent leakage. For \textbf{LongBench-v2}, we report results on the full evaluation set. Detailed statistics and examples are provided in Appendix~\ref{app:datasets}. For retrieval, we follow the default LaRA configuration, using a chunk size of 600 with 100-token overlap, the \texttt{gte-multilingual-base} embedding model and \texttt{gte-multilingual-reranker-base} reranker~\citep{zhang2024mgte}, with vector--rerank split ratio 0.5 and rerank size 5.

\paragraph{Evaluation Metrics}
We report three metrics: (i) \textbf{Route Accuracy} — the fraction of routing decisions that match the \emph{ideal label} (formalized in Sec.~\ref{sec:learn}). The ideal label selects the decision \(y\in\{RAG,LC\}\) that maximizes QA performance \(U\); in ties, it defaults to \(RAG\) to prefer lower cost. Thus, higher Route Accuracy reflects better effectiveness--efficiency alignment. (ii) \textbf{QA Score} — downstream answer quality. Because the original LaRA binary metric is too coarse to distinguish superficial vs.\ substantively complete answers, we adopt a 4-point rubric for finer resolution (examples in Appendix~\ref{app:case_study_metric}). (iii) \textbf{LC Selection Rate} — the proportion of queries routed to the costly \(LC\) path, used as a proxy for computational cost.

\paragraph{Baselines}
We compare against: (a) Always-RAG and Always-LC (fixed strategies); (b) Self-Route~\citep{DBLP:conf/emnlp/Li00MB24selfroute} using the original prompt (Appendix~\ref{app:prompts}); and (c) Pre-Route in three variants that we report separately throughout: Pre-Route (Large: DeepSeek-R1, Qwen3-235B, prompt-only), Pre-Route (Small: Qwen3-1.7B, prompt-only),
and Pre-Route (Distilled: Distilled-Qwen3-1.7B). All Pre-Route prompts are fixed across models. For all applicable methods we report both thinking \textbf{[T]} and no-thinking \textbf{[N]} modes. Always-RAG is excluded from LC Rate ranking (LC=0\%), while Always-LC is excluded from QA ranking due to prohibitive computational costs.

\paragraph{Models}
Our experiments employ the latest available models from the recent Qwen3 series, spanning the 235B, 30B, 8B, 4B, and 1.7B scales. Qwen3-235B and Qwen3-30B are from the July 2025 version. Our fine-tuned student model is referred to as Distilled-Qwen3-1.7B. We also include DeepSeek-R1 (May 2025 version) and the proprietary model Qwen-Max (January 2025 version).


\subsection{Main Results on LaRA Benchmark}
Tab.~\ref{tab:wide-In-distribution-larawide} summarizes the LaRA in-domain results. Note that \textbf{top-two results (bold/underline) consistently fall in our red shaded Pre-Route methods.} Overall, Pre-Route consistently surpasses Self-Route in QA score and route accuracy across all settings, while also reducing the reliance on LC, proving its robustness across various answer model sizes and thinking modes. All reported improvements are statistically significant (paired t-test, $p < 0.01$, Cohen's $d$: 0.19--0.26). We also provide more case studies on why sometimes Self-Route fails on LaRA in Appendix~\ref{app:selfroutefail}.

\textbf{Large model prompt-only routing.} Using Qwen3-235B and DeepSeek-R1, we find that Pre-Route (Prompt-only) significantly outperforms Self-Route across answer models. Structured reasoning prompts alone can reliably elicit latent routing ability, yielding both higher route accuracy and QA score. This suggests that LLMs implicitly learn the comparative strengths of RAG and LC during pretraining---e.g., when information is concentrated vs.\ scattered---and structured prompts elicit this knowledge proactively rather than relying on post-hoc failure signals.

\textbf{Transferability to small models.} When directly applied to Qwen3-1.7B with identical prompts, performance degrades: reasoning chains and routing decisions become unstable. For example, under the Qwen3-235B [T] answer model, Qwen3-1.7B routing lags far behind the large-model routers in accuracy despite comparable QA scores, highlighting that small routers cannot directly inherit complex planning behavior. Notably, small models' errors are systematically tilted toward the ``safer'' LC option (74.3\% of errors), suggesting a \emph{correct but unstable} prior that distillation stabilizes rather than creates.

\textbf{Distillation on in-domain data.} Using Pre-Route (Large) as teacher, we distill into a lightweight Qwen3-1.7B student with reasoning-chain supervision on LaRA. The distilled model successfully acquires structured routing logic, approaching the teacher's QA performance while achieving a lower LC selection rate. For instance, with the strongest answer model Qwen-Max [N], the distilled router reaches clearly higher accuracy at substantially reduced LC usage compared to Self-Route. These results verify that large-model routing logic can be effectively compressed and transferred. Notably, the non-thinking distilled router already performs well, making it the most resource-efficient configuration. We additionally compare against traditional ML classifiers in Appendix~\ref{app:ml}, confirming Pre-Route's gains extend beyond metadata signals. A fine-grained breakdown by document type and difficulty (Appendix~\ref{app:breakdown}) shows consistent gains across all categories, with the largest improvement on hallucination tasks (Acc 0.22$\to$0.75).

\begin{table*}[t]
  \vspace{-10pt}
  \centering
  \large
  \caption{LongBench-v2 (OOD, rerank size = 7): Pre-Route results;
  \textbf{Best} and \underline{Second} per answer model. \emph{Pre-Route Variants:} Large: R1 (DeepSeek-R1) or Q235B (Qwen3-235B); Small: Q1.7B (Qwen3-1.7B); Distilled: D-Q1.7B (Distilled-Qwen3-1.7B).}
  \vspace{-12pt}
  \begin{tcolorbox}[colback=red!4, frame hidden, sharp corners, boxrule=0pt, left=1pt, right=1pt, top=1pt, bottom=1pt, before skip=0pt, after skip=0pt]
  \scriptsize Red-shaded entries denote our recommended \textbf{Pre-Route(Large)} and \textbf{Pre-Route(Distilled)}.
  \end{tcolorbox}
  \label{tab:wide-ood-lb7wide}
  \resizebox{0.92\textwidth}{!}{%
  \begin{tabular}{l|rrr|rrr|rrr|rrr|rrr|rrr}
  \toprule
  \multirow{2}{*}{\begin{tabular}[c]{@{}l@{}}\textbf{Answer Model}\\\textbf{Router Model}\end{tabular}} & \multicolumn{3}{c|}{Qwen3-1.7B [N]} & \multicolumn{3}{c|}{Qwen3-1.7B [T]} & \multicolumn{3}{c|}{Qwen3-4B [N]} & \multicolumn{3}{c|}{Qwen3-4B [T]} & \multicolumn{3}{c|}{Qwen3-8B [N]} & \multicolumn{3}{c}{Qwen3-8B [T]} \\
   & \textbf{QA$\uparrow$} & \textbf{LC(\%)$\downarrow$} & \textbf{Acc$\uparrow$} & \textbf{QA$\uparrow$} & \textbf{LC(\%)$\downarrow$} & \textbf{Acc$\uparrow$} & \textbf{QA$\uparrow$} & \textbf{LC(\%)$\downarrow$} & \textbf{Acc$\uparrow$} & \textbf{QA$\uparrow$} & \textbf{LC(\%)$\downarrow$} & \textbf{Acc$\uparrow$} & \textbf{QA$\uparrow$} & \textbf{LC(\%)$\downarrow$} & \textbf{Acc$\uparrow$} & \textbf{QA$\uparrow$} & \textbf{LC(\%)$\downarrow$} & \textbf{Acc$\uparrow$} \\
  \cmidrule(l{-1pt}r{-1pt}){1-19}
  \cmidrule(l{-1pt}r{-1pt}){1-19}
  Always-LC (Baseline)  & 0.29 & 100.0 & 0.33 & 0.28 & 100.0 & 0.31 & 0.30 & 100.0 & 0.28 & 0.36 & 100.0 & 0.34 & 0.40 & 100.0 & 0.33 & 0.40 & 100.0 & 0.34 \\
  Always-RAG (Baseline)  & 0.34 & 0.0 & 0.67 & 0.34 & 0.0 & 0.69 & 0.33 & 0.0 & 0.72 & 0.38 & 0.0 & 0.66 & 0.36 & 0.0 & 0.67 & 0.39 & 0.0 & 0.66 \\
  Self-Route (Baseline) & \textbf{0.34} & 31.1 & 0.68 & \underline{0.33} & 26.0 & 0.69 & 0.32 & \underline{22.0} & \textbf{0.72} & 0.37 & 30.6 & 0.64 & 0.37 & 24.8 & 0.68 & 0.39 & 28.6 & 0.65 \\
  \cmidrule(l{-1pt}r{-1pt}){1-19}
  \rowcolor{red!4}
  Pre-Route (R1) [T] & 0.32 & 24.8 & 0.71 & \underline{0.33} & \underline{22.5} & \underline{0.72} & \textbf{0.34} & 24.4 & \underline{0.69} & \textbf{0.41} & 25.9 & \underline{0.69} & \underline{0.38} & \underline{19.5} & \underline{0.71} & \textbf{0.42} & 22.7 & \textbf{0.71} \\
  \rowcolor{red!4}
  Pre-Route (Q235B) [N] & 0.32 & 38.1 & 0.61 & \textbf{0.34} & 39.1 & 0.61 & \underline{0.33} & 39.4 & 0.57 & \underline{0.39} & 39.6 & 0.59 & \underline{0.38} & 32.2 & 0.63 & \underline{0.41} & 31.3 & 0.64 \\
  \rowcolor{red!4}
  Pre-Route (Q235B) [T] & 0.32 & 29.7 & 0.68 & \underline{0.33} & 29.4 & 0.69 & \textbf{0.34} & 27.6 & 0.67 & \underline{0.39} & 31.0 & 0.65 & \underline{0.38} & 26.2 & 0.66 & \underline{0.41} & 24.0 & \underline{0.70} \\
  Pre-Route (Q1.7B) [N] & \textbf{0.34} & 37.7 & 0.62 & 0.30 & 36.4 & 0.60 & 0.32 & 39.4 & 0.58 & 0.36 & 36.8 & 0.57 & \underline{0.38} & 40.1 & 0.57 & 0.39 & 41.3 & 0.57 \\
  Pre-Route (Q1.7B) [T] & \underline{0.33} & 42.2 & 0.56 & \underline{0.33} & 43.9 & 0.56 & \underline{0.33} & 40.5 & 0.57 & 0.37 & 40.0 & 0.57 & \textbf{0.39} & 45.9 & 0.55 & \underline{0.41} & 43.1 & 0.55 \\
  \rowcolor{red!4}
  Pre-Route (D-Q1.7B) [N] & \textbf{0.34} & \textbf{6.8} & \textbf{0.84} & \underline{0.33} & \textbf{7.7} & \textbf{0.81} & \textbf{0.34} & 23.9 & \underline{0.69} & 0.37 & \underline{22.9} & 0.68 & 0.37 & \textbf{19.1} & \textbf{0.72} & 0.40 & \textbf{20.4} & \underline{0.70} \\
  \rowcolor{red!4}
  Pre-Route (D-Q1.7B) [T] & \underline{0.33} & \underline{7.0} & \underline{0.83} & \underline{0.33} & \textbf{7.7} & \textbf{0.81} & \underline{0.33} & \textbf{20.3} & \underline{0.69} & 0.38 & \textbf{19.7} & \textbf{0.71} & 0.36 & 20.4 & 0.69 & 0.40 & \underline{21.5} & 0.69 \\
  \cmidrule(l{-1pt}r{-1pt}){1-19}
  \cmidrule(l{-1pt}r{-1pt}){1-19}
  \multirow{2}{*}{\begin{tabular}[c]{@{}l@{}}\textbf{Answer Model}\\\textbf{Router Model}\end{tabular}} & \multicolumn{3}{c|}{Qwen3-30B [N]} & \multicolumn{3}{c|}{Qwen3-30B [T]} & \multicolumn{3}{c|}{Qwen3-235B [N]} & \multicolumn{3}{c|}{Qwen3-235B [T]} & \multicolumn{3}{c|}{DeepSeek-R1 [T]} & \multicolumn{3}{c}{Qwen-Max [N]} \\
   & \textbf{QA$\uparrow$} & \textbf{LC(\%)$\downarrow$} & \textbf{Acc$\uparrow$} & \textbf{QA$\uparrow$} & \textbf{LC(\%)$\downarrow$} & \textbf{Acc$\uparrow$} & \textbf{QA$\uparrow$} & \textbf{LC(\%)$\downarrow$} & \textbf{Acc$\uparrow$} & \textbf{QA$\uparrow$} & \textbf{LC(\%)$\downarrow$} & \textbf{Acc$\uparrow$} & \textbf{QA$\uparrow$} & \textbf{LC(\%)$\downarrow$} & \textbf{Acc$\uparrow$} & \textbf{QA$\uparrow$} & \textbf{LC(\%)$\downarrow$} & \textbf{Acc$\uparrow$} \\
  \cmidrule(l{-1pt}r{-1pt}){1-19}
  \cmidrule(l{-1pt}r{-1pt}){1-19}
  Always-LC (Baseline)  & 0.40 & 100.0 & 0.34 & 0.36 & 100.0 & 0.51 & 0.47 & 100.0 & 0.54 & 0.52 & 100.0 & 0.52 & 0.53 & 100.0 & 0.43 & 0.48 & 100.0 & 0.52 \\
  Always-RAG (Baseline)  & 0.42 & 0.0 & 0.66 & 0.24 & 0.0 & 0.49 & 0.46 & 0.0 & 0.46 & 0.45 & 0.0 & 0.48 & 0.49 & 0.0 & 0.57 & 0.44 & 0.0 & 0.48 \\
  Self-Route (Baseline) & \underline{0.42} & 29.5 & \underline{0.66} & 0.28 & 40.6 & 0.54 & 0.46 & 57.2 & 0.46 & \textbf{0.50} & 46.6 & 0.55 & \underline{0.54} & 35.4 & 0.62 & \underline{0.46} & 49.7 & 0.51 \\
  \cmidrule(l{-1pt}r{-1pt}){1-19}
  \rowcolor{red!4}
  Pre-Route (R1) [T] & 0.40 & \textbf{25.2} & \textbf{0.67} & 0.25 & \textbf{25.5} & \underline{0.62} & \textbf{0.49} & \textbf{25.3} & \textbf{0.72} & \underline{0.48} & \textbf{24.0} & \textbf{0.65} & 0.53 & 26.2 & \textbf{0.70} & \underline{0.46} & 25.8 & \textbf{0.69} \\
  \rowcolor{red!4}
  Pre-Route (Q235B) [N] & 0.40 & 36.8 & 0.58 & 0.28 & 35.1 & 0.60 & \textbf{0.49} & 39.0 & 0.62 & \textbf{0.50} & 38.1 & 0.59 & \textbf{0.55} & 40.3 & 0.62 & \textbf{0.47} & 39.1 & 0.61 \\
  \rowcolor{red!4}
  Pre-Route (Q235B) [T] & 0.41 & 29.2 & \underline{0.66} & 0.26 & 28.8 & 0.60 & \textbf{0.49} & 29.2 & \underline{0.68} & \underline{0.48} & 28.3 & \underline{0.63} & \underline{0.54} & 36.4 & 0.62 & \textbf{0.47} & 29.8 & 0.67 \\
  Pre-Route (Q1.7B) [N] & 0.40 & 48.1 & 0.51 & \underline{0.29} & 42.2 & 0.53 & \underline{0.47} & 36.0 & 0.61 & 0.47 & 35.7 & 0.57 & 0.53 & 41.3 & 0.56 & \underline{0.46} & 42.1 & 0.56 \\
  Pre-Route (Q1.7B) [T] & 0.40 & 44.8 & 0.52 & \textbf{0.30} & 43.9 & 0.56 & 0.46 & 34.4 & 0.60 & \textbf{0.50} & 40.5 & 0.57 & 0.53 & 40.8 & 0.59 & 0.45 & 42.3 & 0.54 \\
  \rowcolor{red!4}
  Pre-Route (D-Q1.7B) [N] & \textbf{0.43} & \underline{26.9} & \textbf{0.67} & 0.27 & 28.5 & \underline{0.62} & \underline{0.47} & \underline{26.4} & 0.67 & \textbf{0.50} & 28.8 & 0.61 & 0.52 & \underline{18.9} & 0.67 & \underline{0.46} & \underline{25.2} & \textbf{0.69} \\
  \rowcolor{red!4}
  Pre-Route (D-Q1.7B) [T] & \underline{0.42} & 28.3 & 0.65 & 0.28 & \underline{27.4} & \textbf{0.64} & 0.46 & 28.0 & 0.66 & \underline{0.48} & \underline{26.9} & 0.62 & 0.52 & \textbf{17.5} & \underline{0.69} & 0.44 & \textbf{24.4} & \underline{0.68} \\
  \cmidrule(l{-1pt}r{-1pt}){1-19}
  \end{tabular}}

  \vspace{-12pt}
\end{table*}

\subsection{Generalization on Out-of-Distribution Data}

\textbf{Out-of-distribution evaluation.} Table~\ref{tab:wide-ood-lb7wide} reports representative LongBench-v2 results, compared to LaRA, LongBench-v2 is out-of-domain not only in data distribution but also in task format and evaluation protocol: LaRA uses open-ended QA with graded scores (0–4), whereas LongBench-v2 reformulates tasks as four-choice multiple-choice questions(MCQ) with binary correctness. This makes Self-Route appear more competitive in QA scores, since its weakness of sometimes routing partially correct answers (as in Appendix~\ref{app:selfroutefail}) is no longer penalized—under binary MCQ evaluation the router only needs to ensure correctness rather than differentiate answer quality. Nevertheless, Pre-Route remains more cost-efficient than Self-Route with substantially low LC usage, effectively offloading ``easy'' samples to maximize compute efficiency. Note that LC Rate reduction is significant ($p < 0.01$).

\begin{figure}[h]
  \vspace{-8pt}
  \centering
  \includegraphics[width=0.68\linewidth]{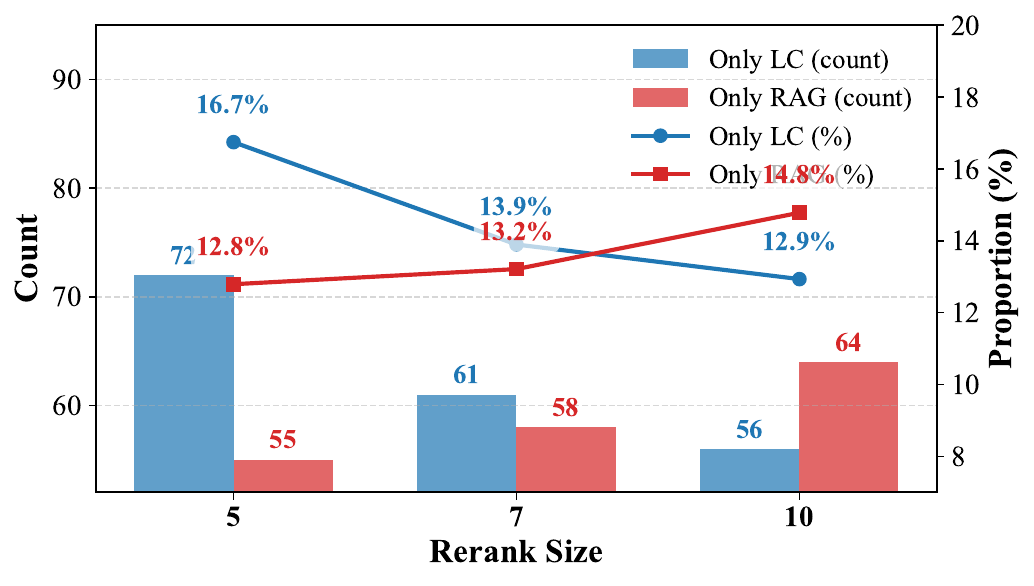}
  \vspace{-8pt}
   \caption{Qwen3-235B: Number of questions answered correctly \emph{only} by LC vs. \emph{only} by RAG.}
  \label{fig:lb_reranksize}
  \vspace{-10pt}
\end{figure}
\vspace{-2pt}

\textbf{Robustness to retrieval configuration.}
Fig.~\ref{fig:lb_reranksize} illustrates the LC vs. RAG balance shift as rerank size increases. Table~\ref{tab:wide-ood-lb7wide} (rerank size 7) and Appendix~\ref{app:additional-longbench-results} (sizes 5, 10) demonstrate Pre-Route's effectiveness across models. Pre-Route (Distilled) remains stable across all configurations, showing minimal sensitivity to retriever changes and consistently achieving higher accuracy with lower LC usage.

\vspace{-12pt}

\subsection{Pre-Route Guidelines Ablation}

\begin{table}[h]

\vspace{-6pt}
\centering
\caption{Ablation analysis. Removing reasoning steps reduces accuracy or increases LC rate, highlighting the necessity of the full Pre-Route design.}
\label{tab:ablation}
\resizebox{0.7\linewidth}{!}{
\begin{tabular}{lccc}
\toprule
\textbf{Variant} & \textbf{QA Score} & \textbf{LC Rate (\%)} & \textbf{Acc} \\
\midrule
\textbf{Pre-Route (full)}      & \textbf{3.38} & 20.7 & \textbf{0.68} \\
- No Reflection       & 3.33 & 20.8 & 0.65 \\
- No Decision Rules   & \textbf{3.38} & 45.3 & 0.57 \\
- No Step1            & 3.33 & \textbf{10.1} & \textbf{0.68} \\
- No Step2            & 3.37 & 27.0 & 0.66 \\
- No Step3            & 3.35 & 18.2 & 0.66 \\
- No Step4            & 3.31 & 17.2 & 0.66 \\
- No Step5            & 3.31 & 21.2 & 0.66 \\
- No Step6            & 3.35 & 24.3 & 0.67 \\
\bottomrule
\end{tabular}}
\vspace{-12pt}
\end{table}
To understand why the structured reasoning chain is effective, we conduct ablations by removing individual steps. As shown in Table~\ref{tab:ablation}, the \textbf{full guideline} achieves the best balance: highest accuracy with a moderate LC rate. Removing components leads either to QA Score drops or inflated LC usage, both stemming from incorrect route choices.

\textbf{Decision rules:} Dropping them causes the router to over-rely on LC, greatly increasing cost. \textbf{Reflection:} Removal slightly lowers accuracy, suggesting it helps refine borderline cases. \textbf{Step~1/2 (Task \& Distribution):} Without them, the router misjudges LC necessity—either overusing or underusing it. \textbf{Step~3–6 (Feasibility \& Trade-offs):} Each prevents misallocation; omitting them causes the router to use LC when RAG suffices or vice versa.

\subsection{\revised{Robustness to Metadata}}
\label{app:metadata}

Pre-Route is not built on the assumption of a fully curated metadata schema. Instead, it deliberately relies on \emph{low-cost metadata} that are almost always present in realistic long-context systems (e.g., document length, a leading snippet), while coarser attributes such as titles or task tags tend to be missing or noisy in practice. We therefore treat length and head as a minimal interface, and focus our analysis on metadata that may not be present.

Importantly, Pre-Route treats the leading snippet as a \emph{soft prior} rather than a hard rule. Decisions emerge from \emph{joint reasoning} over the query and structure, enabling the model to look beyond misleading signals. Appendix~\ref{app:selfroutefail} demonstrates this robustness: Pre-Route correctly selects LC based on distributional cues, even when the head snippet appears to be a sufficient local fact.

Table~\ref{tab:meta-ablation} compares performance across three metadata settings: \textbf{Full-Meta} (all fields), \textbf{Head-only} (length/head only), and \textbf{Generated-Meta} (metadata inferred by Qwen3-1.7B). Results confirm that Pre-Route outperforms the Self-Route baseline even in the \textbf{Head-only} setting, while \textbf{Generated-Meta} effectively closes the gap to \textbf{Full-Meta}, demonstrating robust routing using only readily available or lightweight inferred signals.

\begin{table}[t]
\vspace{-5pt}
\centering
\caption{Effect of metadata settings (Head-only, Generated-Meta, Full-Meta) on routing performance on the LaRA test split.}
\label{tab:meta-ablation}
\resizebox{\columnwidth}{!}{%
\begin{tabular}{l l l c c c}
\toprule
Answer Model & Router Model & Setting & QA$\uparrow$ & LC(\%)$\downarrow$ & Acc$\uparrow$ \\
\midrule
DeepSeek-R1  & Pre-Route(Q235B) [N]   & Head-only      & 3.42 & 20.3 & 0.68 \\
             &                        & Generated-Meta & 3.45 & 20.6 & 0.70 \\
             &                        & Full-Meta      & 3.47 & 27.2 & 0.70 \\
             \cdashline{2-6}
             & Pre-Route(D-Q1.7B) [N] & Head-only      & 3.43 & 18.5 & 0.69 \\
             &                        & Generated-Meta & 3.47 & \textbf{15.8} & \textbf{0.73} \\
             &                        & Full-Meta      & \textbf{3.51} & 20.6 & \textbf{0.73} \\
             \cdashline{2-6}

             & Self-Route (Baseline)  &                & 3.36 & 31.4 & 0.52 \\
\midrule

Qwen3-235B   & Pre-Route(Q235B) [N]   & Head-only      & 3.39 & 23.0 & 0.65 \\
             &                        & Generated-Meta & \textbf{3.44} & 23.2 & \textbf{0.70} \\
             &                        & Full-Meta      & 3.43 & 27.2 & 0.67 \\
             \cdashline{2-6}
             & Pre-Route(D-Q1.7B) [N] & Head-only      & 3.40 & 21.5 & 0.66 \\
             &                        & Generated-Meta & 3.43 & \textbf{20.4} & \textbf{0.70} \\
             &                        & Full-Meta      & 3.43 & 22.7 & 0.69 \\
             \cdashline{2-6}
             & Self-Route (Baseline)  &                & 3.34 & 33.9 & 0.52 \\
\bottomrule

\end{tabular}}
\vspace{-18pt}
\end{table}

\vspace{-12pt}

\section{Conclusion}
\vspace{-6pt}
We introduced \textsc{Pre-Route}, a proactive routing framework that elicits LLMs' latent routing ability to choose between RAG and long-context before answering. Using lightweight structured reasoning, Pre-Route improves routing accuracy and QA quality while reducing LC usage. Experiments on LaRA and LongBench-v2 show Pre-Route outperforms Self-Route with better accuracy–cost trade-offs. Through distillation, small models inherit planning ability for efficient plug-and-play deployment. These results confirm routing ability is latent but can be reliably elicited. \textbf{Future Work} includes extending to multi-modal or open-domain scenarios, hybrid routing.

\section*{Limitations}

Our work has several limitations. First, Pre-Route's effectiveness depends on the quality of metadata (e.g., document type, task type). While we demonstrate robustness with minimal metadata (length + head only), performance may degrade when metadata is completely absent or highly unreliable. Second, the routing decision is binary (RAG vs. LC), which may not capture more nuanced strategies such as hybrid approaches or adaptive retrieval granularity. Third, our evaluation focuses on English-language benchmarks; generalization to other languages and domains requires further investigation. Fourth, the distillation process requires a strong teacher model and high-quality filtered data, which may limit accessibility for some researchers. Finally, while Pre-Route reduces LC usage, the optimal LC selection threshold may vary across different deployment scenarios and cost constraints.

\section*{Ethical Considerations}

This work does not involve human subjects, sensitive personal data, or applications that may raise direct ethical concerns. Our experiments are conducted on publicly available datasets (LaRA and LongBench-v2), and we adhere to responsible AI research practices throughout this study. However, we acknowledge several broader considerations. First, routing systems like Pre-Route could potentially be used to optimize the cost of automated content generation at scale, which may impact certain professional sectors. Second, the system inherits the biases and limitations present in the underlying LLMs used for routing and answering. Third, the retrieval-augmented approach may propagate inaccuracies if the retrieved sources contain errors or misinformation. We encourage users of our system to implement appropriate content filtering and human oversight for high-stakes applications. Code and models are released for research purposes only.

\bibliography{custom}

\appendix

\section{Appendix}

\subsection{Related Work}

\paragraph{Long-context LLMs.}
Improving the ability of LLMs to handle long inputs has long been a key challenge. Early work reduced computational cost by modifying attention ~\citep{DBLP:journals/corr/abs-2004-05150_longformer, DBLP:conf/naacl/GuoAUONSY22_longt5} or compressing inputs ~\citep{DBLP:conf/acl/JiangWL0L0Q24_lonngllm}, while others explored distillation ~\citep{DBLP:conf/acl/HsiehLYNFRKLP23_lesssmall} and cascaded mechanisms ~\citep{DBLP:journals/tmlr/ChenZ024_reducecost}. Recently, frontier models such as GPT-5~\citep{openai2025introducinggpt5}, Gemini-2.5~\citep{DBLP:journals/corr/abs-2507-06261_gemini2.5}, Qwen-3~\citep{DBLP:journals/corr/abs-2505-09388_qwen3report}, and DeepSeek~\citep{deepseek2025r10528} have extended context windows to 128K tokens or way beyond. However, due to the quadratic complexity of Transformers, long-context reasoning remains expensive and unstable under ultra-long inputs, e.g., the ``lost-in-the-middle'' problem ~\citep{DBLP:journals/tacl/LiuLHPBPL24_lostinmiddle}.

\paragraph{Retrieval-Augmented Generation (RAG).}
Another parallel line is RAG~\citep{DBLP:conf/nips/LewisPPPKGKLYR020_ragnlp}), which augments LLMs with external retrieval. RAG shows strong performance in language modeling~\citep{DBLP:conf/iclr/KhandelwalLJZL20_memrag, DBLP:conf/naacl/ShiMYS0LZY24_ragreplug} and open-domain QA~\citep{DBLP:conf/icml/GuuLTPC20_ragpretrain, DBLP:conf/eacl/IzacardG21_opdomain}, at much lower cost than modeling entire documents~\citep{DBLP:conf/icml/BorgeaudMHCRM0L22_ragtrillion}. Subsequent work proposed answer correction~\citep{DBLP:journals/corr/abs-2401-15884_correct_rag}, critique~\citep{DBLP:conf/iclr/AsaiWWSH24_selfrag}, verification~\citep{DBLP:conf/naacl/LiZLYSQ24_llretrieval}, and adaptive retrieval~\citep{DBLP:conf/emnlp/WangLSL23_selfguide_rag, DBLP:conf/emnlp/ChengLLZYSLS0Q24_unirag, DBLP:conf/naacl/JeongBCHP24_adaptrag}, enhancing robustness in knowledge-intensive scenarios. Although retrieval introduces some overhead, RAG pipelines are modular and can be deployed locally.

\paragraph{Benchmarking RAG vs. LC.}
Evaluation has drawn attention as well. A number of benchmarks have been proposed, including both synthetic stress tests~\citep{kamradt2023needle,DBLP:journals/corr/abs-2404-06654_ruler,DBLP:conf/coling/SongZL25_countstar} and task-oriented datasets~\citep{DBLP:conf/acl/BaiLZL0HDLZHDTL24_longbench,DBLP:conf/acl/AnG0ZLZKQ24_leval,DBLP:conf/naacl/DasigiLBCSG21_qasper,narrativeqa,DBLP:conf/naacl/PangPJNPCPMT0B22_quality}. These studies revealed performance degradation as context length grows~\citep{DBLP:conf/acl/LevyJG24_inputlen_impact,DBLP:journals/corr/abs-2404-06654_ruler} and phenomena such as lost-in-the-middle~\citep{DBLP:journals/tacl/LiuLHPBPL24_lostinmiddle}. Nevertheless, most benchmarks do not exceed 128K tokens, limiting their ability to fully assess ultra-long context capabilities. This has motivated larger-scale benchmarks such as \textbf{LaRA}~\citep{DBLP:journals/corr/abs-2502-09977_lara}, \textbf{LongBench-v2}~\citep{bai2024longbench2}, and other related efforts~\citep{DBLP:conf/acl/ZhangCHXCH0TW0024_infbench,DBLP:conf/emnlp/0002IEBL23_zeroscroll,DBLP:conf/coling/DongTLZW24_bamboo,DBLP:conf/acl/LiWZZ24_loogle,wang2024loong}, which better match current LLM capabilities.

\paragraph{The LaRA benchmark.}
LaRA~\citep{DBLP:journals/corr/abs-2502-09977_lara} provides a more systematic comparison of RAG and LC in modern LLMs. It contains test cases across four tasks—\textit{localization, comparison, reasoning, hallucination detection}—and three types of natural long documents: \textit{novels, papers, reports}. Experiments on \textbf{7 open-source} and \textbf{4 proprietary} models show clear complementarity:
\begin{itemize}
    \item \textbf{Model capability:} weaker models rely more on RAG, while stronger ones (e.g., GPT-5, Gemini 2.5) perform better with LC.
    \item \textbf{Task type:} LC excels in reasoning and comparison, while RAG is more robust for hallucination detection and refusals.
    \item \textbf{Text structure:} structured texts (e.g., reports, papers) favor LC, while novels make RAG more cost-effective.
    \item \textbf{Context position \& retrieval granularity:} LC suffers from ``lost-in-the-middle,'' while RAG can cross-span retrieve, but is sensitive to chunk size and number.
\end{itemize}

\paragraph{Routing strategies and Self-Route.}
To combine both paradigms, routing-based methods have been explored. The most representative is Self-Route~\citep{DBLP:conf/emnlp/Li00MB24selfroute}, which adopts a \textit{failure-driven} mechanism: the model first answers using RAG, then falls back to LC if it outputs \textit{unanswerable}. This avoids explicit task recognition, is simple to implement, and works reasonably well in some cases. However, it has clear limitations: (1) reliance on retrieval failure signals without proactive task or information awareness; (2) extra cost from executing RAG before fallback; (3) misrouting due to over- or under-confidence; (4) lack of interpretability in decision logic. Optimizing \textbf{efficiency and performance} remains a central challenge for future routing methods.

\subsection{Optimality of the Ideal Label}
\label{app:proof-cost-minimal}

\begin{proposition}[Optimality of the Ideal Label]
Consider the ideal label defined in Eq.~\ref{eq:ideal-label}, which prescribes
LC if $U(\mathrm{LC};q,\mathcal D) > U(\mathrm{RAG};q,\mathcal D)$
and RAG otherwise (i.e., whenever $U_{\mathrm{LC}} \le U_{\mathrm{RAG}}$).
This rule ensures optimal QA performance by construction,
and among all decision rules $\pi$ that also achieve optimal performance,
the ideal label minimizes the expected cost by defaulting to
the lower-cost RAG in tie cases.
\end{proposition}

\begin{proof}
\textbf{Step 1 (Performance).}
By definition (Eq.~\ref{eq:ideal-label}), the ideal label selects the option with maximal performance: LC when $U_{\mathrm{LC}} > U_{\mathrm{RAG}}$, and RAG otherwise.
Thus no alternative strategy can yield higher performance on any instance, so QA performance is optimal.

\textbf{Step 2 (Cost).}
The cost of a single query can be written as
\[
C_{\text{answer}} = \mathbf{1}\{y=\mathrm{LC}\}\,C_{\mathrm{LC}}
+ \mathbf{1}\{y=\mathrm{RAG}\}\,C_{\mathrm{RAG}},
\]
where $C_{\mathrm{LC}} > C_{\mathrm{RAG}} \ge 0$.
Suppose an alternative strategy $\pi'$ makes the same decisions as the ideal label on strictly-better cases, but deviates on a subset $\mathcal{T}$ of tie cases by choosing LC instead of RAG.
For each $i \in \mathcal{T}$, performance remains equal, but cost strictly increases by $\Delta C_i = C_{\mathrm{LC}} - C_{\mathrm{RAG}} > 0$.
Hence
\[
\mathbb{E}[C(\pi') - C(\pi)]
= \sum_{i\in\mathcal{T}} p_i \Delta C_i \;>\; 0,
\]
where $p_i$ is the probability of case $i$.

\textbf{Conclusion.}
The ideal label therefore ensures optimal QA performance by construction,
and among all decision rules $\pi$ that also achieve optimal performance,
the ideal label minimizes the expected cost by defaulting to
the lower-cost RAG in tie cases.
\qed
\end{proof}

\subsection{Datasets Details}
\label{app:datasets}
We provide additional details of the datasets used in our experiments.


\begin{table}[htbp]
\centering
\caption{Statistics of LongBench-v2 datasets.}
\resizebox{0.8\linewidth}{!}{
\begin{tabular}{lrr}
\toprule
\textbf{Category} & \textbf{\#data} & \textbf{Length} \\
\midrule
I. Single-Document QA        & 175 & 51k--96k \\
II. Multi-Document QA        & 125 & 15k--129k \\
III. Long In-context Learning& 81 & 61k--132k \\
IV. Dialogue Understanding   &  39 & 25k--77k \\
V. Code Repository           &  50 & 167k \\
VI. Structured Data          &  33 & 49k--52k \\
\midrule
\textbf{Total}               & 503 & up to 167k \\
\bottomrule
\end{tabular}}
\vspace{-12pt}
\end{table}

\begin{table}[htbp]
\centering
\caption{Statistics of LaRA.}\resizebox{0.85\linewidth}{!}{
\begin{tabular}{lcccc}
\toprule
\textbf{Context} & \textbf{Location} & \textbf{Reasoning} & \textbf{Comparison} & \textbf{Hallucination} \\
\midrule
\multicolumn{5}{c}{\textbf{32k length}} \\
Novel      & 25673 & 25908 & 25681 & 25433 \\
Financial  & 27548 & 27531 & 27546 & 27527 \\
Paper      & 28078 & 28088 & 27708 & 28081 \\
\# of Cases&   276 &   230 &   151 &   230 \\
\midrule
\multicolumn{5}{c}{\textbf{128k length}} \\
Novel      & 96452 & 96226 & 95903 & 96182 \\
Financial  & 92684 & 92831 & 92830 & 92812 \\
Paper      & 93911 & 93818 & 94731 & 93890 \\
\# of Cases&   489 &   374 &   198 &   378 \\
\bottomrule
\end{tabular}
}
\vspace{-4pt}
\end{table}

\begin{table}[htbp]
\vspace{-4pt}
\centering
\caption{Statistics of LaRA distribution of the test set (0.7:0.1:0.2 split).}
\resizebox{0.9\linewidth}{!}{
\begin{tabular}{lcccc}
\toprule
\textbf{Context} & \textbf{Location} & \textbf{Reasoning} & \textbf{Comparison} & \textbf{Hallucination} \\
\midrule
\multicolumn{5}{c}{\textbf{32k length}} \\
Novel      & 16 & 13 & 10 & 14 \\
Financial  & 33 & 24 & 14 & 20 \\
Paper      & 10 & 11 &  6 & 12 \\
\# of Cases& 59 & 48 & 30 & 46 \\
\midrule
\multicolumn{5}{c}{\textbf{128k length}} \\
Novel      & 50 & 35 & 15 & 32 \\
Financial  & 19 & 16 & 10 & 20 \\
Paper      & 27 & 27 & 17 & 27 \\
\# of Cases& 96 & 78 & 42 & 79 \\
\midrule
\multicolumn{5}{c}{\textbf{Overall (32k + 128k)}} \\
Novel      & 66 & 48 & 25 & 46 \\
Financial  & 52 & 40 & 24 & 40 \\
Paper      & 37 & 38 & 23 & 39 \\
\# of Cases& 155 & 126 & 72 & 125 \\
\bottomrule
\end{tabular}
}
\end{table}
\FloatBarrier

\subsection{Additional Robustness Results on LongBench-v2 (rerank size = 5 and 10)}
\label{app:additional-longbench-results}

To provide comprehensive evaluation of Pre-Route's robustness across different retrieval configurations, we present detailed experimental results for rerank sizes 5 and 10 in this appendix. These results complement the representative findings shown in Table~\ref{tab:wide-ood-lb7wide} and further demonstrate Pre-Route's consistent performance advantages across varying retrieval settings.

\begin{table*}[htbp]
\vspace{-4pt}
\centering
\LARGE
\caption{LongBench-v2 (OOD, rerank size = 5): Pre-Route results;
\textbf{Best} and \underline{Second} per answer model. \emph{Pre-Route Variants:} Large: R1 (DeepSeek-R1) or Q235B (Qwen3-235B); Small: Q1.7B (Qwen3-1.7B); Distilled: D-Q1.7B (Distilled-Qwen3-1.7B).}
\vspace{-4pt}
\label{tab:wide-ood-lb5wide}
\resizebox{1\textwidth}{!}{%
\begin{tabular}{l|rrr|rrr|rrr|rrr|rrr|rrr}
\toprule
\multirow{2}{*}{\begin{tabular}[c]{@{}l@{}}\textbf{Answer Model}\\\textbf{Router Model}\end{tabular}} & \multicolumn{3}{c|}{Qwen3-1.7B [N]} & \multicolumn{3}{c|}{Qwen3-1.7B [T]} & \multicolumn{3}{c|}{Qwen3-4B [N]} & \multicolumn{3}{c|}{Qwen3-4B [T]} & \multicolumn{3}{c|}{Qwen3-8B [N]} & \multicolumn{3}{c}{Qwen3-8B [T]} \\
 & \textbf{QA$\uparrow$} & \textbf{LC(\%)$\downarrow$} & \textbf{Acc$\uparrow$} & \textbf{QA$\uparrow$} & \textbf{LC(\%)$\downarrow$} & \textbf{Acc$\uparrow$} & \textbf{QA$\uparrow$} & \textbf{LC(\%)$\downarrow$} & \textbf{Acc$\uparrow$} & \textbf{QA$\uparrow$} & \textbf{LC(\%)$\downarrow$} & \textbf{Acc$\uparrow$} & \textbf{QA$\uparrow$} & \textbf{LC(\%)$\downarrow$} & \textbf{Acc$\uparrow$} & \textbf{QA$\uparrow$} & \textbf{LC(\%)$\downarrow$} & \textbf{Acc$\uparrow$} \\
\cmidrule(l{-1pt}r{-1pt}){1-19}
\cmidrule(l{-1pt}r{-1pt}){1-19}
Always-LC (Baseline)  & 0.29 & 100.0 & 0.38 & 0.28 & 100.0 & 0.31 & 0.30 & 100.0 & 0.30 & 0.37 & 100.0 & 0.40 & 0.40 & 100.0 & 0.40 & 0.40 & 100.0 & 0.37 \\
Always-RAG (Baseline)  & 0.33 & 0.0 & 0.62 & 0.34 & 0.0 & 0.69 & 0.31 & 0.0 & 0.70 & 0.34 & 0.0 & 0.60 & 0.34 & 0.0 & 0.60 & 0.37 & 0.0 & 0.63 \\
Self-Route (Baseline) & \textbf{0.32} & 34.9 & 0.63 & \textbf{0.33} & 28.1 & 0.68 & 0.30 & 22.7 & 0.69 & \textbf{0.37} & 33.8 & 0.65 & \textbf{0.38} & 31.3 & 0.64 & \textbf{0.41} & 32.3 & 0.67 \\
\cmidrule(l{-1pt}r{-1pt}){1-19}
\rowcolor{red!4}
Pre-Route (R1) [T] & 0.30 & 24.8 & 0.68 & \textbf{0.33} & 23.5 & \underline{0.71} & \underline{0.31} & 24.7 & 0.69 & 0.35 & 24.5 & \textbf{0.70} & \underline{0.37} & \underline{19.6} & \textbf{0.72} & 0.39 & \textbf{19.9} & \textbf{0.71} \\
\rowcolor{red!4}
Pre-Route (Q235B) [N] & \underline{0.31} & 39.7 & 0.58 & 0.31 & 37.8 & 0.60 & \underline{0.31} & 39.5 & 0.59 & \underline{0.36} & 38.7 & 0.59 & \textbf{0.38} & 31.8 & 0.64 & \underline{0.40} & 33.0 & 0.63 \\
\rowcolor{red!4}
Pre-Route (Q235B) [T] & \underline{0.31} & 29.7 & 0.66 & \underline{0.32} & 29.4 & 0.65 & \textbf{0.32} & 27.0 & 0.69 & 0.34 & 30.1 & 0.64 & \underline{0.37} & 25.8 & 0.68 & 0.38 & 29.6 & 0.63 \\
Pre-Route (Q1.7B) [N] & 0.30 & 37.6 & 0.56 & 0.31 & 38.9 & 0.56 & 0.30 & 41.2 & 0.56 & 0.34 & 38.2 & 0.58 & 0.35 & 39.8 & 0.55 & 0.38 & 38.1 & 0.58 \\
Pre-Route (Q1.7B) [T] & 0.30 & 41.7 & 0.54 & \underline{0.32} & 38.7 & 0.58 & \textbf{0.32} & 42.1 & 0.58 & 0.34 & 40.7 & 0.55 & \textbf{0.38} & 38.9 & 0.61 & 0.37 & 40.0 & 0.54 \\
\rowcolor{red!4}
Pre-Route (D-Q1.7B) [N] & \underline{0.31} & \underline{7.3} & \underline{0.81} & \textbf{0.33} & \underline{8.0} & \textbf{0.80} & \underline{0.31} & \underline{20.8} & \underline{0.70} & 0.34 & \underline{21.1} & \underline{0.67} & \textbf{0.38} & 20.7 & \underline{0.70} & 0.38 & 21.9 & \underline{0.68} \\
\rowcolor{red!4}
Pre-Route (D-Q1.7B) [T] & \textbf{0.32} & \textbf{6.6} & \textbf{0.82} & \textbf{0.33} & \textbf{7.5} & \textbf{0.80} & \underline{0.31} & \textbf{20.2} & \textbf{0.71} & 0.34 & \textbf{20.4} & 0.66 & 0.36 & \textbf{18.9} & 0.69 & 0.36 & \underline{21.2} & \underline{0.68} \\
\cmidrule(l{-1pt}r{-1pt}){1-19}
\cmidrule(l{-1pt}r{-1pt}){1-19}
\multirow{2}{*}{\begin{tabular}[c]{@{}l@{}}\textbf{Answer Model}\\\textbf{Router Model}\end{tabular}} & \multicolumn{3}{c|}{Qwen3-30B [N]} & \multicolumn{3}{c|}{Qwen3-30B [T]} & \multicolumn{3}{c|}{Qwen3-235B [N]} & \multicolumn{3}{c|}{Qwen3-235B [T]} & \multicolumn{3}{c|}{DeepSeek-R1 [T]} & \multicolumn{3}{c}{Qwen-Max [N]} \\
 & \textbf{QA$\uparrow$} & \textbf{LC(\%)$\downarrow$} & \textbf{Acc$\uparrow$} & \textbf{QA$\uparrow$} & \textbf{LC(\%)$\downarrow$} & \textbf{Acc$\uparrow$} & \textbf{QA$\uparrow$} & \textbf{LC(\%)$\downarrow$} & \textbf{Acc$\uparrow$} & \textbf{QA$\uparrow$} & \textbf{LC(\%)$\downarrow$} & \textbf{Acc$\uparrow$} & \textbf{QA$\uparrow$} & \textbf{LC(\%)$\downarrow$} & \textbf{Acc$\uparrow$} & \textbf{QA$\uparrow$} & \textbf{LC(\%)$\downarrow$} & \textbf{Acc$\uparrow$} \\
\cmidrule(l{-1pt}r{-1pt}){1-19}
\cmidrule(l{-1pt}r{-1pt}){1-19}
Always-LC (Baseline)  & 0.40 & 100.0 & 0.40 & 0.36 & 100.0 & 0.56 & 0.46 & 100.0 & 0.60 & 0.52 & 100.0 & 0.54 & 0.53 & 100.0 & 0.49 & 0.48 & 100.0 & 0.59 \\
Always-RAG (Baseline)  & 0.39 & 0.0 & 0.60 & 0.22 & 0.0 & 0.44 & 0.42 & 0.0 & 0.40 & 0.43 & 0.0 & 0.46 & 0.45 & 0.0 & 0.51 & 0.40 & 0.0 & 0.41 \\
Self-Route (Baseline) & 0.40 & 38.0 & 0.61 & \textbf{0.28} & 47.8 & 0.51 & 0.45 & 60.9 & 0.44 & \textbf{0.49} & 52.2 & 0.52 & 0.46 & 44.4 & 0.53 & \textbf{0.45} & 57.4 & 0.46 \\
\cmidrule(l{-1pt}r{-1pt}){1-19}
\rowcolor{red!4}
Pre-Route (R1) [T] & \underline{0.41} & \textbf{25.9} & \textbf{0.70} & 0.25 & \textbf{25.2} & \textbf{0.62} & \textbf{0.47} & \textbf{23.5} & \textbf{0.72} & 0.45 & \textbf{23.8} & \textbf{0.63} & 0.47 & 27.6 & 0.61 & 0.43 & \textbf{25.8} & \textbf{0.68} \\
\rowcolor{red!4}
Pre-Route (Q235B) [N] & \textbf{0.42} & 39.0 & 0.60 & \underline{0.27} & 40.0 & 0.53 & \textbf{0.47} & 37.4 & 0.62 & \underline{0.47} & 41.3 & 0.55 & \textbf{0.51} & 41.8 & 0.56 & \underline{0.44} & 40.0 & 0.59 \\
\rowcolor{red!4}
Pre-Route (Q235B) [T] & \underline{0.41} & 29.2 & 0.67 & 0.24 & \underline{28.5} & 0.58 & 0.45 & 27.4 & \underline{0.67} & \underline{0.47} & \underline{28.5} & \underline{0.61} & \textbf{0.51} & 31.1 & \underline{0.65} & \underline{0.44} & 29.0 & 0.65 \\
Pre-Route (Q1.7B) [N] & 0.38 & 39.9 & 0.57 & \underline{0.27} & 35.3 & 0.58 & 0.45 & 36.7 & 0.61 & \underline{0.47} & 37.1 & 0.55 & 0.47 & 42.3 & 0.53 & \textbf{0.45} & 38.7 & 0.60 \\
Pre-Route (Q1.7B) [T] & 0.40 & 41.0 & 0.57 & \underline{0.27} & 41.2 & 0.54 & 0.44 & 40.5 & 0.56 & 0.46 & 43.0 & 0.55 & 0.47 & 41.3 & 0.57 & \underline{0.44} & 41.8 & 0.56 \\
\rowcolor{red!4}
Pre-Route (D-Q1.7B) [N] & \textbf{0.42} & \underline{27.6} & \underline{0.68} & \underline{0.27} & 31.2 & \underline{0.61} & \underline{0.46} & 27.9 & \underline{0.67} & \underline{0.47} & 28.7 & \underline{0.61} & 0.46 & \textbf{18.4} & 0.64 & \underline{0.44} & \underline{27.3} & \underline{0.66} \\
\rowcolor{red!4}
Pre-Route (D-Q1.7B) [T] & \textbf{0.42} & 28.0 & 0.67 & \textbf{0.28} & 29.5 & \textbf{0.62} & 0.45 & \underline{26.7} & 0.66 & 0.46 & 29.0 & 0.58 & \underline{0.48} & \underline{21.9} & \textbf{0.66} & \textbf{0.45} & 27.7 & \underline{0.66} \\
\cmidrule(l{-1pt}r{-1pt}){1-19}
\end{tabular}}

\vspace{-4pt}
\end{table*}

\begin{table*}[htbp]
\vspace{-4pt}
\centering
\LARGE

\caption{LongBench-v2 (OOD, rerank size = 10): Pre-Route results;
\textbf{Best} and \underline{Second} per answer model. \emph{Pre-Route Variants:} Large: R1 (DeepSeek-R1) or Q235B (Qwen3-235B); Small: Q1.7B (Qwen3-1.7B); Distilled: D-Q1.7B (Distilled-Qwen3-1.7B).}
\vspace{-4pt}
\label{tab:wide-ood-lb10wide}
\resizebox{1\textwidth}{!}{%
\begin{tabular}{l|rrr|rrr|rrr|rrr|rrr|rrr}
\toprule
\multirow{2}{*}{\begin{tabular}[c]{@{}l@{}}\textbf{Answer Model}\\\textbf{Router Model}\end{tabular}} & \multicolumn{3}{c|}{Qwen3-1.7B [N]} & \multicolumn{3}{c|}{Qwen3-1.7B [T]} & \multicolumn{3}{c|}{Qwen3-4B [N]} & \multicolumn{3}{c|}{Qwen3-4B [T]} & \multicolumn{3}{c|}{Qwen3-8B [N]} & \multicolumn{3}{c}{Qwen3-8B [T]} \\
 & \textbf{QA$\uparrow$} & \textbf{LC(\%)$\downarrow$} & \textbf{Acc$\uparrow$} & \textbf{QA$\uparrow$} & \textbf{LC(\%)$\downarrow$} & \textbf{Acc$\uparrow$} & \textbf{QA$\uparrow$} & \textbf{LC(\%)$\downarrow$} & \textbf{Acc$\uparrow$} & \textbf{QA$\uparrow$} & \textbf{LC(\%)$\downarrow$} & \textbf{Acc$\uparrow$} & \textbf{QA$\uparrow$} & \textbf{LC(\%)$\downarrow$} & \textbf{Acc$\uparrow$} & \textbf{QA$\uparrow$} & \textbf{LC(\%)$\downarrow$} & \textbf{Acc$\uparrow$} \\
\cmidrule(l{-1pt}r{-1pt}){1-19}
\cmidrule(l{-1pt}r{-1pt}){1-19}
Always-LC (Baseline)  & 0.29 & 100.0 & 0.32 & 0.27 & 100.0 & 0.30 & 0.30 & 100.0 & 0.26 & 0.36 & 100.0 & 0.30 & 0.40 & 100.0 & 0.27 & 0.40 & 100.0 & 0.28 \\
Always-RAG (Baseline)  & 0.32 & 0.0 & 0.68 & 0.31 & 0.0 & 0.70 & 0.33 & 0.0 & 0.74 & 0.39 & 0.0 & 0.70 & 0.43 & 0.0 & 0.73 & 0.43 & 0.0 & 0.72 \\
Self-Route (Baseline) & \textbf{0.33} & 28.2 & \underline{0.69} & \underline{0.31} & 24.2 & 0.70 & \underline{0.33} & \textbf{19.5} & \textbf{0.74} & 0.38 & 26.0 & \textbf{0.68} & 0.41 & \underline{21.1} & \underline{0.72} & 0.41 & 24.3 & 0.70 \\
\cmidrule(l{-1pt}r{-1pt}){1-19}
\rowcolor{red!4}
Pre-Route (R1) [T] & 0.31 & 25.9 & 0.68 & \textbf{0.32} & 24.8 & 0.71 & \underline{0.33} & 23.6 & 0.70 & \textbf{0.40} & 29.5 & 0.64 & \underline{0.44} & \textbf{19.3} & \textbf{0.74} & \underline{0.45} & \textbf{20.9} & \textbf{0.76} \\
\rowcolor{red!4}
Pre-Route (Q235B) [N] & 0.31 & 39.5 & 0.58 & \underline{0.31} & 40.2 & 0.59 & \underline{0.33} & 38.6 & 0.57 & \underline{0.39} & 40.1 & 0.57 & \textbf{0.45} & 31.7 & 0.66 & \textbf{0.46} & 32.8 & 0.67 \\
\rowcolor{red!4}
Pre-Route (Q235B) [T] & \underline{0.32} & 29.7 & 0.66 & \underline{0.31} & 27.5 & 0.68 & \textbf{0.34} & 29.4 & 0.66 & \underline{0.39} & 30.4 & 0.64 & 0.43 & 27.4 & 0.66 & \underline{0.45} & 27.5 & 0.70 \\
Pre-Route (Q1.7B) [N] & \textbf{0.33} & 39.9 & 0.59 & \underline{0.31} & 34.1 & 0.62 & 0.32 & 35.8 & 0.58 & 0.36 & 38.0 & 0.56 & 0.41 & 40.0 & 0.56 & 0.42 & 36.5 & 0.61 \\
Pre-Route (Q1.7B) [T] & 0.31 & 40.3 & 0.56 & \underline{0.31} & 41.1 & 0.57 & \underline{0.33} & 41.6 & 0.56 & 0.38 & 38.9 & 0.56 & 0.43 & 39.1 & 0.60 & 0.43 & 42.2 & 0.57 \\
\rowcolor{red!4}
Pre-Route (D-Q1.7B) [N] & \textbf{0.33} & \underline{7.6} & \textbf{0.83} & \underline{0.31} & \textbf{7.5} & \textbf{0.82} & \textbf{0.34} & 21.2 & \underline{0.71} & 0.38 & \underline{23.7} & \textbf{0.68} & 0.42 & 23.4 & 0.71 & 0.42 & \underline{21.6} & 0.69 \\
\rowcolor{red!4}
Pre-Route (D-Q1.7B) [T] & \underline{0.32} & \textbf{6.8} & \textbf{0.83} & \underline{0.31} & \underline{9.5} & \underline{0.80} & 0.32 & \underline{20.6} & 0.70 & 0.37 & \textbf{22.4} & \underline{0.66} & 0.42 & 22.0 & \underline{0.72} & 0.44 & \textbf{20.9} & \underline{0.71} \\
\cmidrule(l{-1pt}r{-1pt}){1-19}
\cmidrule(l{-1pt}r{-1pt}){1-19}
\multirow{2}{*}{\begin{tabular}[c]{@{}l@{}}\textbf{Answer Model}\\\textbf{Router Model}\end{tabular}} & \multicolumn{3}{c|}{Qwen3-30B [N]} & \multicolumn{3}{c|}{Qwen3-30B [T]} & \multicolumn{3}{c|}{Qwen3-235B [N]} & \multicolumn{3}{c|}{Qwen3-235B [T]} & \multicolumn{3}{c|}{DeepSeek-R1 [T]} & \multicolumn{3}{c}{Qwen-Max [N]} \\
 & \textbf{QA$\uparrow$} & \textbf{LC(\%)$\downarrow$} & \textbf{Acc$\uparrow$} & \textbf{QA$\uparrow$} & \textbf{LC(\%)$\downarrow$} & \textbf{Acc$\uparrow$} & \textbf{QA$\uparrow$} & \textbf{LC(\%)$\downarrow$} & \textbf{Acc$\uparrow$} & \textbf{QA$\uparrow$} & \textbf{LC(\%)$\downarrow$} & \textbf{Acc$\uparrow$} & \textbf{QA$\uparrow$} & \textbf{LC(\%)$\downarrow$} & \textbf{Acc$\uparrow$} & \textbf{QA$\uparrow$} & \textbf{LC(\%)$\downarrow$} & \textbf{Acc$\uparrow$} \\
\cmidrule(l{-1pt}r{-1pt}){1-19}
\cmidrule(l{-1pt}r{-1pt}){1-19}
Always-LC (Baseline)  & 0.41 & 100.0 & 0.31 & 0.36 & 100.0 & 0.46 & 0.46 & 100.0 & 0.50 & 0.51 & 100.0 & 0.49 & 0.53 & 100.0 & 0.38 & 0.48 & 100.0 & 0.47 \\
Always-RAG (Baseline)  & 0.45 & 0.0 & 0.69 & 0.25 & 0.0 & 0.54 & 0.48 & 0.0 & 0.50 & 0.45 & 0.0 & 0.51 & 0.49 & 0.0 & 0.62 & 0.48 & 0.0 & 0.53 \\
Self-Route (Baseline) & \underline{0.46} & 29.0 & 0.70 & \textbf{0.30} & 35.0 & 0.60 & 0.48 & 54.0 & 0.49 & \textbf{0.50} & 44.0 & 0.57 & 0.50 & 28.9 & 0.63 & 0.48 & 46.2 & 0.53 \\
\cmidrule(l{-1pt}r{-1pt}){1-19}
\rowcolor{red!4}
Pre-Route (R1) [T] & 0.45 & \textbf{22.8} & \textbf{0.73} & 0.27 & \textbf{25.0} & \textbf{0.64} & \textbf{0.51} & \textbf{25.6} & \textbf{0.73} & 0.47 & \textbf{23.3} & \textbf{0.66} & 0.51 & 24.5 & \textbf{0.69} & \underline{0.49} & 25.8 & \textbf{0.70} \\
\rowcolor{red!4}
Pre-Route (Q235B) [N] & \underline{0.46} & 36.3 & 0.62 & \underline{0.29} & 37.9 & 0.58 & \textbf{0.51} & 38.3 & 0.63 & 0.48 & 40.2 & 0.57 & 0.51 & 39.7 & 0.58 & \textbf{0.50} & 39.6 & 0.60 \\
\rowcolor{red!4}
Pre-Route (Q235B) [T] & 0.45 & 29.7 & 0.67 & 0.26 & 31.1 & 0.58 & \textbf{0.51} & \underline{26.3} & \underline{0.71} & 0.48 & 31.0 & 0.61 & 0.50 & 32.4 & 0.61 & \textbf{0.50} & 30.8 & 0.67 \\
Pre-Route (Q1.7B) [N] & 0.45 & 48.0 & 0.53 & \textbf{0.30} & 42.3 & 0.57 & \underline{0.50} & 35.8 & 0.62 & \textbf{0.50} & 36.7 & 0.60 & \textbf{0.53} & 41.7 & 0.59 & 0.48 & 38.7 & 0.59 \\
Pre-Route (Q1.7B) [T] & \underline{0.46} & 43.2 & 0.57 & \textbf{0.30} & 41.1 & 0.57 & \underline{0.50} & 42.3 & 0.57 & \underline{0.49} & 39.5 & 0.56 & 0.50 & 43.6 & 0.52 & 0.48 & 41.1 & 0.57 \\
\rowcolor{red!4}
Pre-Route (D-Q1.7B) [N] & \underline{0.46} & 25.9 & 0.71 & \underline{0.29} & 27.6 & \underline{0.63} & \underline{0.50} & 27.5 & 0.70 & \textbf{0.50} & 30.0 & \underline{0.62} & \textbf{0.53} & \underline{23.0} & \underline{0.66} & 0.48 & \textbf{23.9} & \underline{0.69} \\
\rowcolor{red!4}
Pre-Route (D-Q1.7B) [T] & \textbf{0.47} & \underline{25.2} & \underline{0.72} & 0.28 & \underline{26.9} & 0.61 & \textbf{0.51} & 27.5 & 0.70 & 0.48 & \underline{29.0} & 0.61 & \underline{0.52} & \textbf{19.1} & \textbf{0.69} & 0.48 & \underline{25.6} & \underline{0.69} \\
\cmidrule(l{-1pt}r{-1pt}){1-19}
\end{tabular}}

\vspace{-4pt}
\end{table*}

\newpage

\subsection{Case Study on the Necessity of a Finer-grained Evaluation Metric}

To illustrate why binary evaluation is insufficient, we present a case study highlighting the value of a 4-point scoring rubric. This example shows how two answers may both be judged "correct" under binary metrics, yet differ substantially in completeness and alignment with ground truth. The case study below (Tab.~\ref{app:case_study_metric}) concretely demonstrates how finer-grained evaluation better reflects qualitative differences.

\label{app:case_study_metric}
\begin{tcolorbox}[
  title=Case Study 1: ``Why does Stella repeatedly think their father will `kill' Miles?'',
  colback=gray!5,
  colframe=black,
  colbacktitle=black,
  coltitle=white,
  fonttitle=\bfseries,
  breakable,
  enhanced jigsaw
]

\subsubsection*{1) Initial Analysis: The Query and RAG's Scope}

\textbf{Query:} \textit{"Why does Stella repeatedly think that their father will 'kill' Miles?"}

The presence of the keyword \textbf{repeatedly} is crucial, as it implies that the necessary context is likely scattered across multiple passages, requiring a broader understanding than localized snippets might provide.

\textbf{RAG Retrieved Snippets Analysis:} The retrieved passages contain phrases like \textbf{"Miles was expelled from school"} and \textbf{"dad will be very angry."} These snippets allow a model to construct a direct causal chain \textbf{but inherently lack the historical context needed to address the "repeatedly" aspect of the query}.

\textbf{Ground Truth:} Stella knows that Miles has been expelled from multiple schools, and when she realizes he got kicked out again, she's convinced their father will be furious and punish Miles harshly because it's an ongoing issue, and they have high expectations for him.

\medskip
\hrule 
\medskip

\subsubsection*{2) Comparative Analysis and Scoring}

\textbf{Long-Context Answer (Score: 4/4):}
"Stella repeatedly thinks their father will 'kill' Miles because she knows he was expelled \textbf{again}. She anticipates not only her father's anger but also his \textbf{disappointment}, given the \textbf{high expectations} they have for Miles, and thus fears a \textbf{severe outcome}."

\textbf{RAG Answer (Score: 3/4):}
"Stella thinks their father will 'kill' Miles because Miles was expelled from school, and she knows their father will be very angry and punish him severely."

\textbf{Analysis:}
Under a binary (correct/incorrect) system, both answers would likely be marked "correct" as they identify the core reason (the expulsion). However, this masks a substantial quality gap. To properly evaluate them, we compare them against the key components of the ground truth:
\begin{itemize}
    \item The expulsion is a repeated event (\textbf{again / multiple schools}).
    \item The father has \textbf{high expectations} for Miles.
    \item The father's reaction includes extreme \textbf{anger and disappointment}.
    \item The consequences are expected to be \textbf{severe}.
\end{itemize}

The \textbf{RAG answer} covers the basic causal chain but fails to address the crucial contexts of \textbf{repetition} and \textbf{high expectations}. In contrast, the \textbf{Long-Context answer} successfully incorporates these nuances, aligning almost perfectly with the ground truth.

Our 4-point rubric is designed to capture these critical distinctions. By assigning a higher score to the Long-Context answer, we can accurately reflect its deeper contextual alignment and superior quality.

\textbf{Conclusion:} The Long-Context strategy is demonstrably superior for this query (\texttt{better\_strategy = long\_context}).

\end{tcolorbox}

\subsection{Case Studies on Why Self-Route Sometimes Fails on LaRA}
\label{app:selfroutefail}
Across our case studies, a consistent pattern emerges: \textbf{Self-Route tends to fail at the extremes of confidence}.

\begin{itemize}[leftmargin=1.2em]
  \item \textbf{Over-confidence.} When retrieved chunks appear superficially relevant, Self-Route prematurely concludes that RAG suffices, even if the query requires \emph{distributed reasoning} across multiple parts of the document. In Case~1, it ignored the global cue ``repeatedly'' and thus failed to capture the need for long-context integration.
  \item \textbf{Under-confidence.} Conversely, when the evidence is present but takes the form of a \emph{negative premise check} (``not discussed'') or a \emph{keyword-localized fact} (e.g., a single time point), Self-Route misclassifies the situation as ``unanswerable.'' This leads to unnecessary escalation to long-context models, wasting cost and providing no additional benefit (Cases~2 and~3).
\end{itemize}

In other words, Self-Route lacks a calibrated notion of when retrieved evidence is \emph{sufficient} versus when it is \emph{incomplete}. Our router, by relying on lightweight meta signals (question form, genre, and evidence distribution cues), avoids these extremes: it dispatches \emph{distributed} queries to Long-Context (Case~1), while keeping \emph{localized} or \emph{premise-check} queries on RAG (Cases~2 and~3).
\begin{tcolorbox}[
  title=Case Study 1: ``Why does Stella repeatedly think their father will `kill' Miles?'',
  colback=gray!5,
  colframe=black,
  colbacktitle=black,
  coltitle=white,
  fonttitle=\bfseries,
  breakable,
  enhanced jigsaw
]
\textbf{What was asked?}\\
Explain \emph{why} Stella keeps thinking this (note the word \emph{``repeatedly''}).

\textbf{Where is the evidence?}\\
Scattered across the story: repeated expulsions (\emph{again}), strict family expectations, and the fear of harsh punishment.

\textbf{What did the two answering strategies say?}
\begin{itemize}[leftmargin=1.2em]
  \item \textbf{RAG (picked by Self-Route):} Gave the short cause: Miles was expelled $\Rightarrow$ father angry/punish. \emph{Missed} the ``\emph{again}'' and "high expectations". \textbf{Score: 3}.
  \item \textbf{Long-Context (picked by our router):} Added the missing parts: expelled \emph{again}, family expectations, and "angry + \emph{disappointed}". \textbf{Score: 4}.
\end{itemize}

\textbf{Who routed what? Why?}
\begin{itemize}[leftmargin=1.2em]
  \item \textbf{Self-Route} chose \textbf{RAG} because retrieved chunks looked answerable.
  \item \textbf{Our router} chose \textbf{Long-Context} because \emph{``repeatedly''} means the proof is spread out.
\end{itemize}

\textbf{Why did Self-Route miss?}\\
It treated ``retrieved chunks can answer'' as enough and ignored the global signal \emph{``repeatedly''}.

\textbf{Why our routing was correct.}
Using meta signals only: the question includes \emph{``repeatedly''}, the genre is long-form narrative, and the document fits the window. These cues imply evidence is \emph{distributed} across scenes, so \textbf{Long-Context} is the appropriate choice.

\end{tcolorbox}

\begin{tcolorbox}[
  title=Case Study 2: ``Did the paper discuss ReLU-at-classification for NLP?'',
  colback=gray!5,
  colframe=black,
  colbacktitle=black,
  coltitle=white,
  fonttitle=\bfseries,
  breakable,
  enhanced jigsaw
]
\textbf{What was asked?}\\
Check whether \emph{paper 1} talks about using ReLU as a \emph{classification layer} for \emph{NLP}, compared with traditional models.

\textbf{Where is the evidence?}\\
Right in the abstract/introduction: the paper runs on \textbf{MNIST, FashionMNIST, WDBC} (image/biomedical), \emph{not NLP}.

\textbf{Ground truth.}\\
``\emph{Paper 1 does not address the impact \dots\ on NLP tasks.}''

\textbf{What did the two answering strategies say?}
\begin{itemize}[leftmargin=1.2em]
  \item \textbf{RAG:} Said the paper \emph{does not discuss} NLP (correct). \textbf{Score: 4}.
  \item \textbf{Long-Context:} Same conclusion (also correct). \textbf{Score: 4}.
\end{itemize}

\textbf{Who routed what? Why?}
\begin{itemize}[leftmargin=1.2em]
  \item \textbf{Self-Route} sent it to \textbf{Long-Context} after treating checking retrieved chunks and made``not discussed'' as ``unanswerable''.
  \item \textbf{Our router} picked \textbf{RAG} because this is a \emph{premise check}; the abstract is enough.
\end{itemize}

\textbf{Why did Self-Route miss?}\\
It confused a valid negative answer (\emph{``not discussed''}) with ``unanswerable'', and used more context than needed.

\textbf{Why our routing was correct.}
Based solely on meta information: the query is a \emph{premise check} ("did it discuss …?"), and such evidence is typically \emph{localized} in the abstract/introduction. Under these cues and the "prefer RAG on ties" rule, \textbf{RAG} is recommended for efficiency.

\end{tcolorbox}

\begin{tcolorbox}[
  title=Case Study 3: ``On April 2 when did the whole view fall into darkness?'',
  colback=gray!5,
  colframe=black,
  colbacktitle=black,
  coltitle=white,
  fonttitle=\bfseries,
  breakable,
  enhanced jigsaw
]
\textbf{What was asked?}\\
A single time point on April 2.

\textbf{Where is the evidence? (verbatim)}\\
``\emph{\dots until, \textbf{at five minutes before seven}, the whole surface in view was enveloped in the darkness of night.}''

\textbf{What did the two answering strategies say?}
\begin{itemize}[leftmargin=1.2em]
  \item \textbf{RAG:} ``\emph{at five minutes before seven (6:55)}'' \textbf{Score: 4}.
  \item \textbf{Long-Context:} ``\emph{At five minutes before seven.}'' \textbf{Score: 4}.
\end{itemize}

\textbf{Who routed what? Why?}
\begin{itemize}[leftmargin=1.2em]
  \item \textbf{Self-Route} escalated to \textbf{Long-Context} after saying ``unanswerable''.
  \item \textbf{Our router} chose \textbf{RAG} because this is a \emph{keyword-localizable} fact (date/event/time).
\end{itemize}

\textbf{Why did Self-Route miss?}\\
Even though the answer is in retrieved chunks, somehow Self-Route failed to give answer.

\textbf{Why our routing was correct.}
From meta-only cues: the query asks for a \emph{single time point} with clear date/time keywords that usually appears in \emph{one sentence}. These signals favor \textbf{RAG} for precise, keyword-localizable extraction.

\end{tcolorbox}

\FloatBarrier

\newpage

\subsection{Comparison with Traditional ML Routers}
\label{app:ml}

To isolate the contribution of LLM-based structured reasoning from the predictive value of metadata alone, we train Decision Tree and Random Forest classifiers using the \textbf{same metadata features} as Pre-Route (document length, question length, task type, document type, difficulty level) on the LaRA training set. Results are averaged across 3 answer models (qwen3-4b, qwen3-235b-a22b, deepseek-r1-0528).

\begin{table}[h]
\centering
\caption{Pre-Route SFT vs.\ traditional ML classifiers on LaRA. All methods use identical metadata features. Results averaged across 3 answer models.}
\label{tab:ml-baseline}
\resizebox{\linewidth}{!}{%
\begin{tabular}{lccc}
\toprule
Method & QA Score & LC Rate & Accuracy \\
\midrule
\textbf{Pre-Route SFT} & \textbf{3.256} & \textbf{0.229} & \textbf{0.708} \\
Random Forest (n=8) & 3.106 & 0.317 & 0.641 \\
Decision Tree (depth=6) & 3.014 & 0.414 & 0.572 \\
\bottomrule
\end{tabular}}
\end{table}

Pre-Route outperforms ML baselines by +6.7--13.6\% accuracy over Random Forest and Decision Tree respectively, with lower LC Rate (22.9\% vs.\ 31.7--41.4\%). This shows that: (1) Metadata contains useful routing signals (ML achieves 57--64\% accuracy). (2) Distilled knowledge provides additional value---the +6.7--13.6\% gain comes from structured reasoning captured from large models. (3) ML models struggle on LongBench-v2 because its metadata categories (task\_type, doc\_type) are disjoint from LaRA's, while Pre-Route leverages LLM semantic understanding for zero-shot transfer.

\begin{figure*}[h]
  \centering
  \begin{minipage}{0.65\linewidth}
    \centering
    \includegraphics[width=\linewidth]{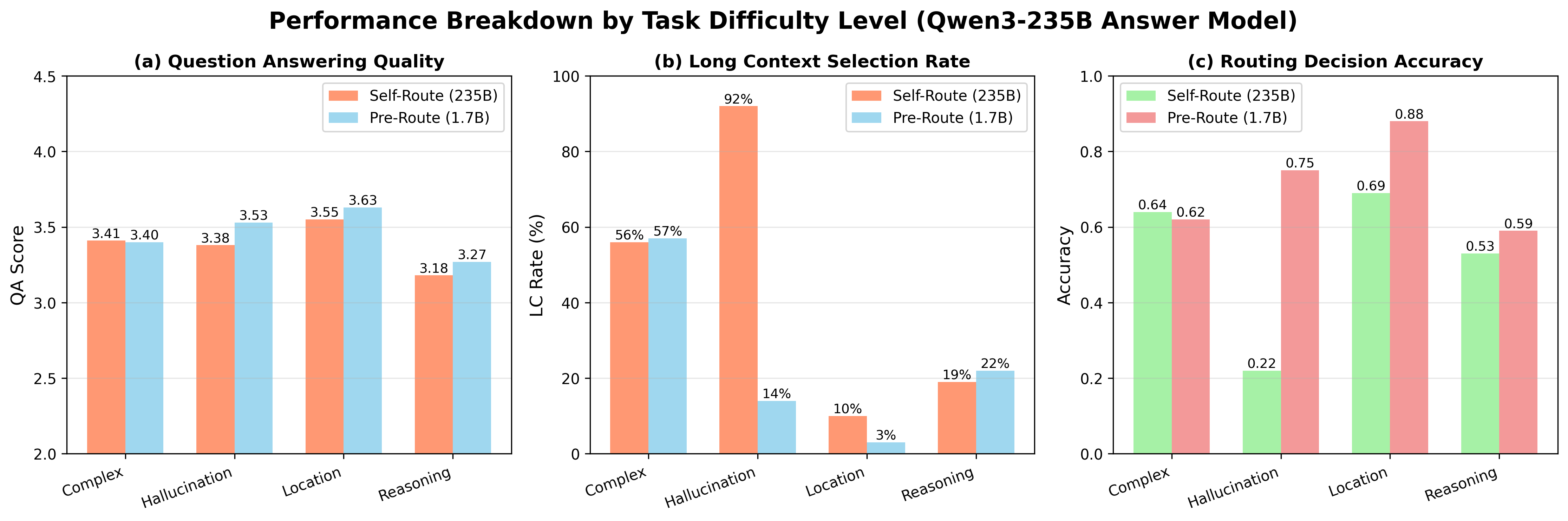}
  \end{minipage}
  \hfill
  \begin{minipage}{0.28\linewidth}
    \centering
    \includegraphics[width=\linewidth]{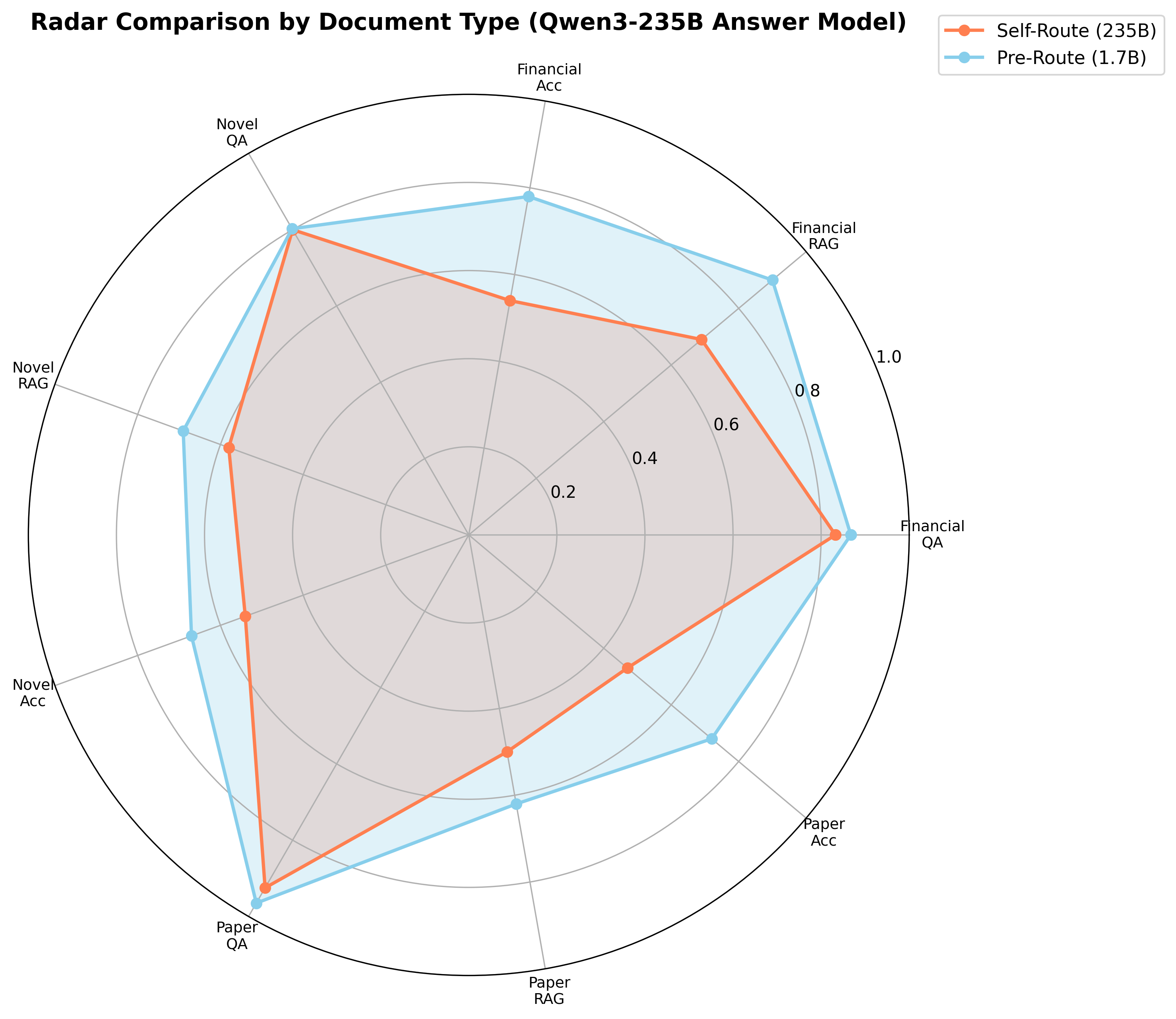}
  \end{minipage}
  \caption{Left: QA Score and Route Accuracy breakdown by difficulty level. Right: Radar chart comparing Self-Route and Pre-Route across task types (Qwen3-235B answer model).}
  \label{fig:breakdown-level}
  \label{fig:radar-comparison}
  \end{figure*}
  
\subsection{Performance Breakdown by Document Type and Difficulty}
\label{app:breakdown}

\paragraph{By document type.}
Table~\ref{tab:breakdown-doc} shows Pre-Route's advantage across document types (Paper, Financial, Novel). Improvements are consistent, with the largest QA gains on Paper (+0.16) and Financial (+0.14) documents.

\begin{table}[h]
\centering
\caption{Performance breakdown by document type (Qwen3-235B answer model, D-Q1.7B router).}
\label{tab:breakdown-doc}
\resizebox{\linewidth}{!}{%
\begin{tabular}{lcccccc}
\toprule
& \multicolumn{2}{c}{Self-Route} & \multicolumn{2}{c}{Pre-Route} & & \\
\cmidrule(lr){2-3} \cmidrule(lr){4-5}
Doc Type & QA & Acc & QA & Acc & $\Delta$QA & $\Delta$Acc \\
\midrule
Paper & 3.70 & 0.47 & 3.86 & 0.72 & \textbf{+0.16} & \textbf{+0.25} \\
Financial & 3.33 & 0.54 & 3.47 & 0.78 & \textbf{+0.14} & \textbf{+0.24} \\
Novel & 3.20 & 0.54 & 3.21 & 0.67 & +0.01 & \textbf{+0.13} \\
\bottomrule
\end{tabular}}
\end{table}

\paragraph{By difficulty level.}
Table~\ref{tab:breakdown-diff} reveals where Pre-Route's advantage concentrates. The most striking finding is on Hallucination tasks: Self-Route's accuracy drops to 0.22 (vs.\ Pre-Route's 0.75), indicating its failure-driven ``unanswerable'' trigger is particularly unreliable for hallucination-sensitive queries. On Location tasks, Pre-Route achieves 0.88 accuracy, confirming that structured reasoning correctly identifies keyword-localizable queries that RAG can handle efficiently.

\begin{table}[h]
\centering
\caption{Performance breakdown by difficulty level (Qwen3-235B answer model, D-Q1.7B router).}
\label{tab:breakdown-diff}
\resizebox{\linewidth}{!}{%
\begin{tabular}{lcccccc}
\toprule
& \multicolumn{2}{c}{Self-Route} & \multicolumn{2}{c}{Pre-Route} & & \\
\cmidrule(lr){2-3} \cmidrule(lr){4-5}
Level & QA & Acc & QA & Acc & $\Delta$QA & $\Delta$Acc \\
\midrule
Hallucination & 3.38 & 0.22 & 3.53 & 0.75 & \textbf{+0.15} & \textbf{+0.53} \\
Location & 3.55 & 0.69 & 3.63 & 0.88 & +0.08 & \textbf{+0.19} \\
Reasoning & 3.18 & 0.53 & 3.27 & 0.59 & +0.09 & +0.06 \\
Complex & 3.41 & 0.64 & 3.40 & 0.62 & -0.01 & -0.02 \\
\bottomrule
\end{tabular}}
\end{table}

\subsection{Small Model Error Analysis}
\label{app:error-analysis}

We analyze samples where the 235B router succeeds but the prompt-only 1.7B router fails, revealing the main error patterns for small models:

\textbf{Over-conservatism (74.3\% of errors):} For simple location or factual queries, the 1.7B model is intimidated by long documents or technical jargon, safely but inefficiently defaulting to LC when RAG is perfectly sufficient. By difficulty level, these over-conservative errors concentrate on Location (46.7\%) and Hallucination (29.4\%) tasks.

\textbf{Root cause:} The fundamental gap lies in global semantic understanding vs.\ surface heuristics. The 235B model understands the query's deep structure and resists surface-level traps, while the 1.7B model relies on shallow cues (document length, jargon presence). This capability gap explicitly justifies our distillation pipeline: it transfers stable decision boundaries to small models, rather than relying on their brittle zero-shot reasoning.

\subsection{Training Details}
\label{app:training}

The distilled router (Qwen3-1.7B) is trained with the following hyperparameters:
\begin{itemize}[nosep]
\item Optimizer: AdamW (lr=2e-5, weight\_decay=0.01)
\item Schedule: Cosine LR with 10\% warmup steps
\item Epochs: 6
\item Batch size: 8 (gradient accumulation=4)
\item Max sequence length: 4000 tokens
\item Precision: BF16
\end{itemize}

\subsection{Reproducibility Statement}
\label{app:reproduce}
We have made efforts to ensure reproducibility. The main paper describes the proposed method, model configurations, and evaluation protocols. Additional training details, hyperparameters, and extended results are provided in the appendix.

\subsection{Example Prompts}
We provide the full prompts used in our experiments.
Figs.~\ref{fig:self_route_prompt_lara} and \ref{fig:self_route_prompt} shows the original Self-Route prompt, while Figs~\ref{fig:hallu_eval_prompt} and \ref{fig:standard_eval_prompt} present evaluation prompts adapted from LaRA for hallucination detection and general answer comparison.
Finally, Fig.~\ref{fig:complete_decision_prompt} gives the complete Pre-Route prompt.

\label{app:prompts}
\vspace{-10pt}
\begin{figure*}[h]
\centering
\begin{tcolorbox}[
  title=Complete Prompt for Self-Route,
  colback=gray!5,
  colframe=black,
  colbacktitle=black,
  coltitle=white,
  fonttitle=\bfseries,
  breakable,
  enhanced jigsaw,
  width=\linewidth
]

\begin{lstlisting}[breaklines=true, breakatwhitespace=true, basicstyle=\ttfamily, columns=flexible, lineskip=2pt, aboveskip=0pt, belowskip=0pt, language=]
Answer the question based on the given passages. Only give me the answer and do not output any other words. If the question cannot be answered based on the information in the article, write "unanswerable". The following are given passages. {rag_result} Question: {query} Answer:
\end{lstlisting}
\end{tcolorbox}
\caption{The complete Self-Route Prompt for LaRA. Taken directly from its original paper~\citep{DBLP:conf/emnlp/Li00MB24selfroute}}
\label{fig:self_route_prompt_lara}
\end{figure*}
\vspace{-22pt}
\begin{figure*}[h]
\centering
\begin{tcolorbox}[
  title=Complete Prompt for Self-Route,
  colback=gray!5,
  colframe=black,
  colbacktitle=black,
  coltitle=white,
  fonttitle=\bfseries,
  breakable,
  enhanced jigsaw,
  width=\linewidth
]

\begin{lstlisting}[breaklines=true, breakatwhitespace=true, basicstyle=\ttfamily, columns=flexible, lineskip=2pt, aboveskip=0pt, belowskip=0pt, language=]
Read the following text and answer briefly. {rag_chunks} Now, answer the following question based on the above text, only give me the answer and do not output any other words. If the question cannot be answered based on the information in the article, write "unanswerable". Question: {query} Answer:
\end{lstlisting}
\end{tcolorbox}
\caption{The complete Self-Route Prompt for Longbench-v2 multiple-choice. Taken directly from its original paper~\citep{DBLP:conf/emnlp/Li00MB24selfroute}}
\label{fig:self_route_prompt}
\end{figure*}

\begin{figure*}[h]
\centering
\begin{tcolorbox}[
  title=Evaluation Prompt for Hallucination Detection (Adapted from LaRA),
  colback=gray!5,
  colframe=black,
  colbacktitle=black,
  coltitle=white,
  fonttitle=\bfseries,
  breakable,
  enhanced jigsaw,
  width=\linewidth
]
\begin{lstlisting}[breaklines=true, breakatwhitespace=true, columns=flexible, lineskip=1pt, aboveskip=0pt, belowskip=0pt, language=]
You will be given a question, a groundtruth answer, and two answers from AI assistants (Assistant A and Assistant B).
This question may not be directly answerable from the source text.
Your task:
1. Check whether each assistant is hallucinating. If the answer is **consistent** with the groundtruth, and the question is **not present in the original text**, it is considered **not hallucinating**.
2. For each answer, assign a quality score from 1 to 4:
   - 4 = Fully correct (not hallucinating and consistent)
   - 3 = Mostly correct (minor detail missing or uncertain)
   - 2 = Partially correct (some correct points, some hallucination)
   - 1 = Hallucinated, irrelevant, or incorrect
3. Write a brief comparative analysis.
4. Decide who is better: "A", "B", or "Tie"
Return JSON in this format:
{
  "analysis": "...",
  "score_a": int,
  "score_b": int,
  "better": "A" | "B" | "Tie"
}
[Question]
{query}
[Groundtruth Answer]
{label}
[Assistant A's Answer]
{pred_a}
[Assistant B's Answer]
{pred_b}
Start your evaluation:
\end{lstlisting}
\end{tcolorbox}

\newpage
\caption{An adapted prompt from the LaRA benchmark, specifically refined for evaluating hallucination. It provides a detailed scoring rubric to distinguish between factual consistency and fabrication.}
\label{fig:hallu_eval_prompt}
\end{figure*}

\begin{figure*}[h]
\centering
\begin{tcolorbox}[
  title=Standard Evaluation Prompt for Answer Comparison (Adapted from LaRA),
  colback=gray!5,
  colframe=black,
  colbacktitle=black,
  coltitle=white,
  fonttitle=\bfseries,
  breakable,
  enhanced jigsaw,
  width=\linewidth
]
\begin{lstlisting}[breaklines=true, breakatwhitespace=true, language=]
You will be given a question, a groundtruth answer, and two AI assistant answers.

Your task:
1. Judge each answer's factual accuracy and completeness relative to the groundtruth.
2. For each answer, assign a score from 1 to 4:
   - 4 = Fully correct (covers all key points, no factual or logical flaws)
   - 3 = Mostly correct (core is right, minor flaws or omissions)
   - 2 = Partially correct (some accurate parts, but notable errors or gaps)
   - 1 = Incorrect or hallucinated (serious flaws, irrelevant, fabricated)
3. Provide a brief comparison explaining your scoring.
4. Choose which assistant is better: "A", "B", or "Tie"

Return your result in the following JSON format:
{
  "analysis": "...",
  "score_a": int,
  "score_b": int,
  "better": "A" | "B" | "Tie"
}

[Question]
{query}

[Groundtruth Answer]
{label}

[Assistant A's Answer]
{pred_a}

[Assistant B's Answer]
{pred_b}

Start your evaluation:
\end{lstlisting}
\end{tcolorbox}
\caption{A general-purpose evaluation prompt adapted from LaRA for detailed comparison. This version focuses on factual accuracy, completeness, and overall quality against a groundtruth answer.}
\label{fig:standard_eval_prompt}
\end{figure*}

\newpage

\begin{figure*}[h]
\vspace{-12pt}
\centering

\begin{tcolorbox}[
  title=Complete Prompt for Pre-Route,
  colback=gray!5,
  colframe=black,
  colbacktitle=black,
  coltitle=white,
  fonttitle=\bfseries,
  breakable,
  enhanced jigsaw,
  width=\linewidth,
  top=1pt,
  bottom=1pt,
  left=2pt,
  right=2pt
]

\begin{lstlisting}[breaklines=true, breakatwhitespace=true, basicstyle=\scriptsize, columns=flexible, lineskip=2pt, aboveskip=0pt, belowskip=0pt]
## Inputs
- query: "{query}"
- task_type: {task_type}
- document_title: "{doc_title}"
- document_type: {doc_type}
- document_length: {length_str} ({doc_len_tokens} tokens)
- answering_model: {model}
- answering_max_window: {max_window_tokens} tokens
- document_fits_window: {document_fits}
- document_head content: "{doc_head}..."
## RAG Configuration
- chunk_size: {chunk_size} tokens
- chunk_overlap: {chunk_overlap} tokens
- embed_model: {embed_model}
- rerank_model: {rerank_model}
- vector_ratio: {vector_ratio}
- rerank_size: {rerank_size}
## Instructions
You are tasked with choosing the most appropriate strategy -- **RAG**, **LONG_CONTEXT** -- for answering the user query, based on the characteristics of the query, document, and model.
Please complete the `<thinking>` block below. In each step, provide a clear judgment and explain how it affects your strategy choice.
1. `<step1>`: Identify the question type (e.g., factual, reasoning, comparison, judgment, etc.) and the document type (e.g., book, article, report). How do these affect the need for deep context understanding or precise retrieval?
2. `<step2>`: Assess whether the relevant information is likely concentrated in one part of the document or scattered across multiple sections. How does this affect strategy selection?
3. `<step3>`: Evaluate whether the document can fully fit into the context window, based on `document_fits_window` above. If not, how does that impact strategy choice?
4. `<step4>`: Consider whether the query can be answered through keyword-based retrieval (e.g., names, dates), or requires synthesizing implicit logic, analogies, or multi-part reasoning.
5. `<step5>`: Reflect on the model being used (e.g., {model}). Consider both its context window size and model capacity (parameters). Although some models may have large context windows, smaller models may still struggle with effective long-context reasoning due to limited capacity. How does this influence your strategy recommendation?
6. `<step6_efficiency>`: Compare **expected efficiency** of RAG vs LONG_CONTEXT: expected context size, retrieval selectivity, latency, and cost. If quality is likely similar, which strategy is more efficient?
7. `<reflection>`: Based on your reasoning above, state which strategy is more suitable overall -- **RAG**, **LONG_CONTEXT** -- and explain why.
8. `<decision>`: Write your final strategy choice clearly as either `RAG`, `LONG_CONTEXT`.
### Decision Rules
- If both strategies are **equally suitable** or quality difference is **negligible/uncertain**, **prefer RAG for efficiency**.
- Prefer **LONG_CONTEXT** only if (a) the document **fits** in the window **and** (b) the query requires **global, cross-section synthesis** that retrieval would likely miss.
- Prefer **RAG** when the document **does not fit** the window, or when **high-precision snippet retrieval** is likely effective (e.g., names, dates, localized facts), or when **efficiency** is a concern and quality is similar.
## Output Format
<thinking>
  <step1>...</step1>
  <step2>...</step2>
  <step3>...</step3>
  <step4>...</step4>
  <step5>...</step5>
  <step6_efficiency>...</step6_efficiency>
  <reflection>...</reflection>
  <decision>...</decision>
</thinking>
\end{lstlisting}

\end{tcolorbox}

\caption{The complete Pre-Route Prompt.}
\label{fig:complete_decision_prompt}

\end{figure*}

\end{document}